\documentclass{article}

\usepackage{arxiv}

\usepackage{amssymb}
\usepackage{amsthm}
\usepackage{enumerate}
\usepackage[breaklinks=true]{hyperref}
\usepackage{float}
\usepackage{IEEEtrantools}
\usepackage{mathtools}
\usepackage{algorithm}
\usepackage{algorithmic}
\usepackage[utf8]{inputenc} 
\usepackage[T1]{fontenc}    
\usepackage{hyperref}       
\usepackage{url}            
\usepackage{booktabs}       
\usepackage{amsfonts}       
\usepackage{nicefrac}       
\usepackage{microtype}      
\usepackage{lipsum}
\usepackage{graphicx}

\usepackage{csquotes}
\usepackage{xspace}
\usepackage{caption}
\usepackage{subcaption}

\providecommand{\abs}[1]{\lvert#1\rvert}

\DeclarePairedDelimiter\ceil{\lceil}{\rceil}

\DeclareMathOperator*{\argmax}{arg\,max}
\DeclareMathOperator*{\argmin}{arg\,min}

\newtheorem{definition}{Definition}
\newtheorem{lemma}{Lemma}

\title{Classification based on Topological Data Analysis}

\author{
 Rolando Kindelan \\
  Computer Science Department\\ 
  Faculty of Mathematical and Physical Sciences\\
  University of Chile\\ 
  851 Beauchef Av. Santiago de Chile, Chile. \\  
  Center of Medical Biophysics,\\
  Universidad de Oriente, Santiago de Cuba, Cuba.\\
  \texttt{rkindela@dcc.uchile.cl} \\
   \And
 José Frías \\
  Center for Research in Mathematics\\
  Jalisco S/N, Col. Valenciana\\ 
  CP: 36023 Guanajuato, Gto., México\\
  \texttt{friaas@matem.unam.mx} \\
  \And
 Mauricio Cerda \\
  Integrative Biology Program\\
  Institute of Biomedical Sciences\\ 
  Biomedical Neuroscience Institute\\
  Center for Medical Informatics and Telemedicine\\
  Faculty of Medicine\\ 
  Universidad de Chile\\
  1027 Independencia Av., Santiago, Chile.\\
  \texttt{mauricio.cerda@uchile.cl} \\
  \And
 Nancy Hitschfeld \\
  Computer Science Department\\ 
  Faculty of Mathematical and Physical Sciences\\
  University of Chile\\ 
  851 Beauchef Av. Santiago de Chile, Chile \\
  \texttt{nancy@dcc.uchile.cl} \\
}

\begin{document}

\maketitle
\begin{abstract}
Topological Data Analysis (TDA) is an emergent field that aims to discover topological information hidden in a dataset. TDA tools have been commonly used to create filters and topological descriptors to improve Machine Learning (ML) methods. This paper proposes an algorithm that applies TDA directly to multi-class classification problems, even imbalanced datasets, without any further ML stage. The proposed algorithm built a filtered simplicial complex on the dataset. Persistent homology is then applied to guide choosing a sub-complex where unlabeled points obtain the label with most votes from labeled neighboring points. To assess the proposed method, 8 datasets were selected with several degrees of class entanglement, variability on the samples per class, and dimensionality. On average, the proposed TDABC method was capable of overcoming baseline classifiers (wk-NN and k-NN) in each of the computed metrics, especially on classifying entangled and minority classes.
\end{abstract}


\section{Introduction}
\label{sec:intro}

The processing and extraction of information from large and noisy data sets is a challenging problem in Computer Science. The techniques of algebraic topology have gained the attention of scientists for years, giving rise to an emerging research field called Topological Data Analysis (TDA)~\cite{Carlsson:Bulletin, EdelsbrunnerHarer2010}. TDA is an approach to infer the topology underlying a dataset by using combinatorial algebraic structures known as simplicial complexes. TDA also involves the computation of invariant properties from continuous transformations of these simplicial complexes: a process known as persistent homology~\cite{EdelsbrunnerHarer2010}.\par
Over several decades, the high dimensionality of datasets coupled with the combinatorial and continuous character of topology have been problematic issues, making computing persistent homology a challenge that has been addressed by several authors.~Edelsbrunner et al.~\cite{EdelsbrunnerHarer2010} present an efficient algorithm and its visualization as a persistence diagram~\cite{EdelsbrunnerHarer2010, zomorodian_2005}. Carlsson et al.~\cite{Carlsson:Bulletin} strengthened the mathematical foundations and also proposed another visualization tool called persistence Barcodes~\cite{Ghrist2008, Carlsson:Bulletin}. Further developments in the TDA field are derived from those initial works. \par 
As a consequence of their combinatorial nature, the construction and representation of simplicial complexes also represent a challenge. Many works have dealt with efficient construction and representation of filtered simplicial complexes. Data structures and algorithms have been developed~\cite{DBLP:journals/talg/BoissonnatS18, Zom2010, DBLP:journals/algorithmica/BoissonnatST17, DBLP:journals/algorithmica/BoissonnatM14}, and they have mainly focused on the construction of \v{C}ech, Rips and other kinds of simplicial complexes such as: Witness, Alpha, Delaunay, Tangent, and Cover complexes. Theoretical and practical results have been organized as TDA libraries:~GUDHI~\cite{DBLP:journals/algorithmica/BoissonnatM14, gudhi2014}, Dionysus, Ripser, Dipha, Perseus and JavaPlex. A complete benchmark of those libraries can be found in~\cite{Otter2017}.\par


 Regarding the use of TDA for classification, a TDA-based method was used in~\cite{tdaretina} for classifying high-resolution diabetic retinopathy images. They used a preprocessing stage for computing persistent homology to detect topological features encoded into persistence diagrams. A support vector machine (SVM) was used to classify the images according to the persistence descriptors which were used to discriminate between diabetic and healthy patients. \par
Moreover, TDA has been applied to time-series analyses~\cite{TDAapp2017}. One common pipeline is to consider the time-series as a dynamic system and compute the attractors or time-variant of the signal, which creates a manifold around the attractors and turns the signal into phase-domain~\cite{DBLP:journals/corr/VenkataramanRT16, umeda2017}. Persistent Homology or another TDA-tool is applied on these phase-space manifold to create topological descriptors~\cite{timeserie_tda_2016}, and as a final step, a machine learning method is applied such as k-NN, CNN, or SVM. Recently, TDA has been applied in Deep learning to address the interpretability problem~\cite{GunnarGab2018},  to regularize loss functions~\cite{gabrielsson20a}, and to build a persistence layer to consider topological information during learning~\cite{persistencelayer2017, gabrielsson20a}. \par

What all those examples of TDA-applications have in common is that TDA has been used as a preprocessing stage of conventional Machine Learning (ML) algorithms. However, during the TDA pipeline execution, multi-scale relationships among data occur and disappear. From the moment when a multi-scale relationship occurs until it is mixed with another one, it is called persistence. The persistence of many of those relationships is captured and represented by persistence diagrams or barcodes.~Taking advantage of the entire TDA pipeline and not just the result could help address some of the current challenges of supervised and semi-supervised learning, such as imbalanced data classification, identification, and correction of mislabeled data; missing data analysis; and dimensionality reduction.  \par

In this scenario, this paper proposes a methodology to make a TDA pipeline that is able to classify balanced and imbalanced datasets with no ML further stage. The fundamental idea is to provide neighborhoods on a filtered simplicial complex related to a point set (a simplex) or a point as a special case. Those neighborhoods will be, in fact, a sub-complex of the filtered simplicial complex built on the dataset. Persistent homology is used to guide the detection of an appropriate sub-complex from the entire filtration. A labeling process is then made to propagate labels from labeled points to unlabeled points, taking advantage of the simplicial relationships. \par

To illustrate, the proposed method is compared with several baseline classifiers. One of the baseline algorithms is the k-NN algorithm, one of the most popular supervised classification methods. The second baseline method is an enhanced version of k-NN, the weighted k-NN (wk-NN) especially suited for imbalanced datasets. This document is organized into several sections. Section~\ref{scc:methods} presents the mathematical foundations used in this work. Section~\ref{scc:classifier} explains the concepts, algorithms, and methodology of the proposed classification method. Next, Section~\ref{scc:results} describes algorithms, datasets, the experimental protocol, evaluation criteria, and the selected metrics to assess the proposed method performance. In Section~\ref{scc:discussion}, the results and implementation details of the proposed method are explained. Conclusions are presented in Section~\ref{scc:conclusions}.


\section{Mathematical foundations}\label{scc:methods}

In this section, mathematical definitions are introduced (i.e.: simplices, simplicial complex, the \v{C}ech and Rips complexes, the star, and link concepts). Concepts such as persistent homology, filtration, sub-complex, and filtration levels are briefly presented. For more detailed definitions, please see~\cite{EdelsbrunnerHarer2010}. 

\subsection{Simplicial Complexes} \label{scc:math}

Simplicial complexes are combinatorial and algebraic objects which represent a discrete space homotopically equivalent to a data space. Concepts related to simplicial complexes are defined briefly as follows: a q-simplex $ \sigma$ is the convex hull of $q + 1$ affinely independent points $\{s_0, s_1, \cdots, s_q\}\subset\mathbb{R}^n$, $q\leq n$. In this case, the set $\mathcal{V}(\sigma)=\{s_0, s_1, \cdots, s_q\}$ is called the set of \emph{vertices} of $\sigma$ and the simplex $\sigma$ is generated by the set $\mathcal{V}(\sigma)$; this relation will be denoted by $\sigma=[s_{0},s_{1},\dots,s_{q}]$. A q-simplex $\sigma$ has dimension $dim(\sigma)=q$ and it has $\abs{\mathcal{V}(\sigma)} = q + 1$ vertices. Given a q-simplex $\sigma$, a d-simplex $\tau$ with $0\leq d \leq q$ and $\mathcal{V}(\tau)\subseteq \mathcal{V}(\sigma)$ is called a \emph{d-face} of $\sigma$, denoted by $\tau \leq \sigma$,  and $\sigma$ is called a \emph{q-coface} of $\tau$, denoted by $\sigma \geq \tau$.  Note that the 0-faces of a q-simplex $\sigma$ are the elements of $\mathcal{V}(\sigma)$, the 1-faces are line segments with endpoints in $\mathcal{V}(\sigma)$ and so forth.  A q-simplex has $\binom{q+1}{d+1}$ d-faces and $\sum_{d=0}^{q}\binom{q+1}{d+1} = 2^{q+1}$ faces in total. \par 
In order to define homology groups of topological spaces, the notion of simplicial complexes is central:

\begin{definition} \textbf{(Simplicial complex):}
A simplicial complex $\mathcal{K}\subset\mathbb{R}^{n}$ is a finite collection of simplices such that:

\begin{itemize}
    \item $\sigma \in \mathcal{K} \mbox{ and } \tau \leq \sigma \implies \tau \in \mathcal{K} $.
    \item $\sigma_1, \sigma_2 \in \mathcal{K} \implies \sigma_1 \cap \sigma_2 \leq \sigma_1$ and $\sigma_1 \cap \sigma_2 \leq \sigma_2$.        
\end{itemize}
The dimension of $\mathcal{K}$ is $dim(\mathcal{K}) = max\{dim(\sigma) \mid \sigma \in \mathcal{K}\}$. \par
\end{definition}

There are many known simplicial complexes, though two of the most popular are the \v{C}ech~\cite{Ghrist2008, EdelsbrunnerHarer2010} and Vietoris-Rips complexes. In the following definitions the set $B_{\varepsilon}(x)\subset\mathbb{R}^{n}$ should denote the open ball of radius $\varepsilon$ and centered at $x$, namely $B_{\varepsilon}(x)=\{y\in\mathbb{R}^{n}\mid |x-y|<\varepsilon\}$.

\begin{definition}\textbf{(\v{C}ech complex):}\label{def:Cech}
Let $X$ be a finite subspace of $\mathbb{R}^{n}$ and fix $\varepsilon >0$. The \v{C}ech complex $\check{C}ech(\varepsilon)$ is a simplicial complex where the vertices or 0-simplices are the elements of $X$ and a set of vertices $\{v_{0},v_{1},\dots,v_{q}\}$ define a q-simplex if
$$\bigcap_{i=0}^{q}B_{\varepsilon}(v_{i})\neq\emptyset$$
\end{definition}

\begin{definition}\textbf{(Vietoris-Rips or $VR$ complex):}\label{def:VR}
Let $X$ be a finite metric space and fix $\varepsilon >0$. The Vietoris-Rips complex $VR(\varepsilon)$ is a simplicial complex where the 0-simplices are the elements of $X$ and a set of vertices $\{v_{0},v_{1},\dots,v_{q}\}$ define a q-simplex $\sigma=[v_{0},v_{1},\dots,v_{q}]$ if $diam(\sigma) \leq 2\varepsilon$.
\end{definition}

From the above definitions it follows that  $\check{C}ech(\varepsilon) \subseteq VR(\varepsilon) \subseteq \check{C}ech(2\varepsilon)$, where a proof is given in~\cite{EdelsbrunnerHarer2010}, and this relationship is shown in Figure~\ref{fig:cechvr}.
The \v{C}ech complex is intrinsically a high dimensional simplicial complex. From a computational sense, $VR$ complex is more feasible (i.e. lower storage and time complexity) than \v{C}ech, even when the $VR$ complex has more simplices in general. Compared to \v{C}ech, the VR complex does not need to be stored entirely, as it can be stored as a graph and be reconstituted combinatorially~\cite{Ghrist2008}. Even when the results in this paper could be applied to several simplicial complexes with minor changes, this document is focused on \v{C}ech and $VR$ complexes. 

\begin{figure}[H]
\centering
\includegraphics[width=0.5\textwidth]{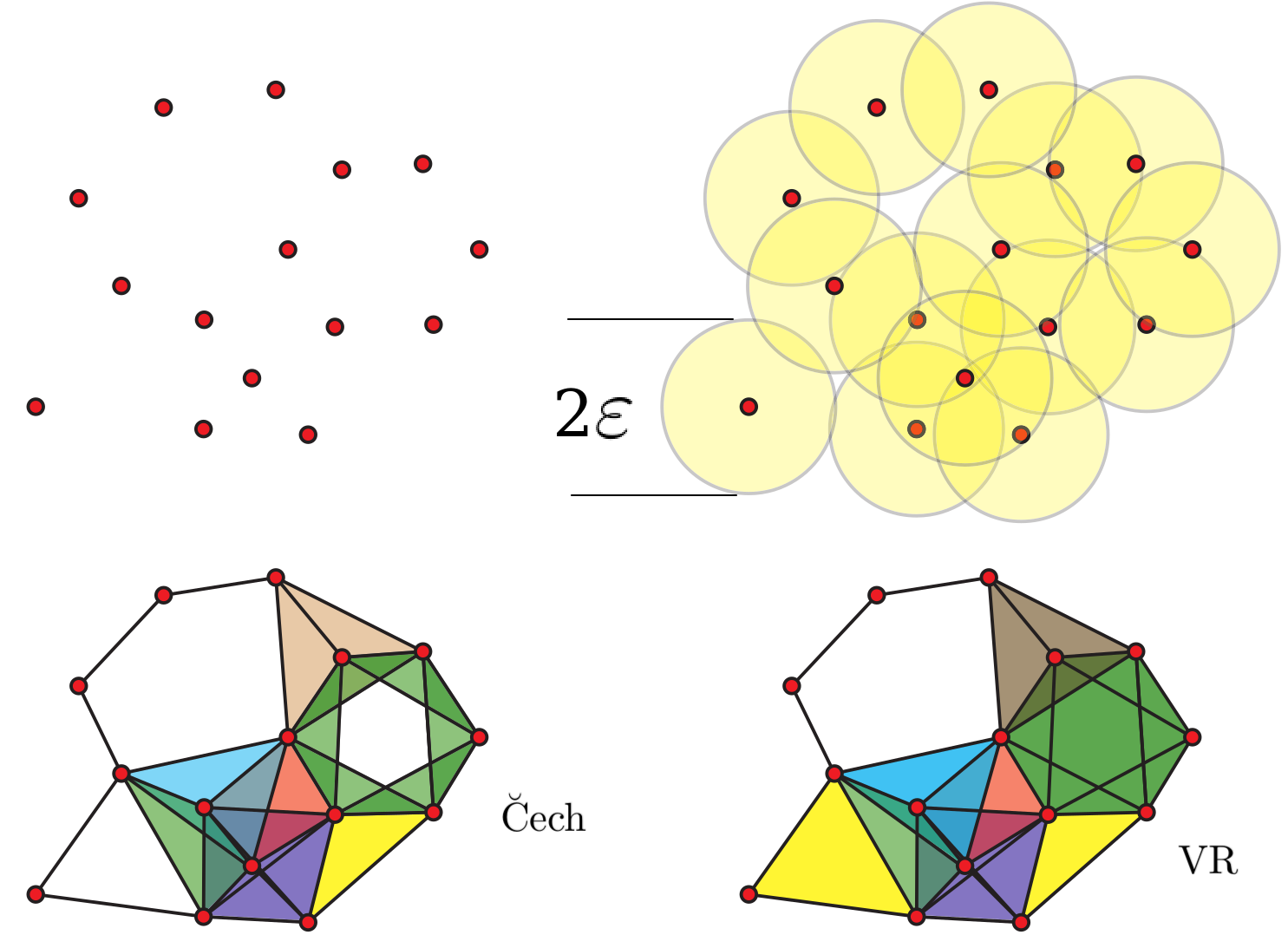}
\caption{
From a point set [upper left] a proximity parameter $\varepsilon$ is applied [upper right] and two complexes were built: a \v{C}ech complex [lower left] and a VR complex [lower right]. \textit{This picture was taken from~\cite{Ghrist2008}}.}\label{fig:cechvr}
\end{figure}
\begin{definition}\textbf{(Star, Closure, Closed Star, and Link):}\label{def:starlink}
   Let $\mathcal{K}$ be a simplicial complex, and $\sigma \in \mathcal{K}$ be a q-simplex. The $star (St)$ of $\sigma$ is the set of all co-faces of $\sigma$ in $\mathcal{K}$~\cite{EdelsbrunnerHarer2010}:
    $$St_\mathcal{K}(\sigma) = \{\tau \in \mathcal{K} \mid \sigma \leq \tau\}$$ 
   
   Let $K$ be a subset of simplices $K \subset \mathcal{K}$. The $closure$ of $K$ is the smallest simplicial complex containing $K$: 
   
   $$Cl_\mathcal{K}(K) = \{\mu \in \mathcal{K} \mid \mu \leq \sigma \text{ for some } \sigma \in K \}$$
   
   The $St_\mathcal{K}(\sigma)$ is not a simplicial complex because of all the missing faces. The smallest simplicial complex that contains $St_\mathcal{K}(\sigma)$ is the closed star (closure of star) of $\sigma$: 
   
   $$\overline{St}_\mathcal{K}(\sigma) = Cl_\mathcal{K}(St_\mathcal{K}(\sigma))$$
   
    The $link (Lk)$ of $\sigma$ is a set of simplices in its closed star that does not share any face with $\sigma$~\cite{EdelsbrunnerHarer2010}: 
    $$Lk_\mathcal{K}(\sigma) = \{\tau \in \overline{St}_\mathcal{K}(\sigma) \mid \tau \cap \sigma = \emptyset\}$$ 
        
\end{definition}

The concept of link of a simplex in a simplicial complex will be important along this paper. For this reason we present two equivalent characterizations of this set:

\begin{lemma}: \label{lemma:002} Let $\mathcal{K}$ be a simplicial complex and $\sigma\in \mathcal{K}$. Then $Lk_\mathcal{K}(\sigma)$ coincide with the sets
\begin{equation}
 A  = \overline{St}_\mathcal{K}(\sigma) \setminus (St_\mathcal{K}(\sigma) \cup Cl_\mathcal{K}(\sigma)), \text{ and}
 \end{equation}
\begin{equation}\label{eq:linklemma}
B=\bigcup_{\mu \in St_{\mathcal{K}}(\sigma)} \{[\mathcal{V}(\mu) \setminus \mathcal{V}(\sigma)]\}
\end{equation}
\end{lemma}

\begin{proof}
Let $\tau$ be a simplex in $Lk_\mathcal{K}(\sigma)$. In particular, $\tau$ does not belong to $St_\mathcal{K}(\sigma)$ nor $ Cl_\mathcal{K}(\sigma)) $ since any simplex in one of these two sets necessarily intersects $\sigma$, then $Lk_\mathcal{K}(\sigma)\subset A$. \\
If $\tau$ is a simplex in $A$, then there exists $\mu\in St_\mathcal{K}(\sigma)$ such that $\tau\leq \mu$ and $(\mathcal{V}(\mu)\setminus\mathcal{V}(\tau))\subset\mathcal{V}(\sigma)$. It follows that 
$\tau=[\mathcal{V}(\mu)\setminus\mathcal{V}(\sigma)]$ and $A\subset B$.\\
Finally, if $\tau\in B$, then $\tau=[\mathcal{V}(\mu)\setminus\mathcal{V}(\sigma)]$ for some $\mu\in St_\mathcal{K}(\sigma)$. It follows that $\tau\in \overline{St}_\mathcal{K}(\sigma)$, but $\tau\cap \sigma =\emptyset$. Then, $B\subset Lk_\mathcal{K}(\sigma)$, and the equivalence of sets is stated.

\end{proof}

Figure~\ref{fig:linkstar} presents an example of $St$ and $Lk$ of point $s_4$ from a given simplicial complex $\mathcal{K}$ build on a point set $S$. 
    
\begin{figure}[H]
 \centering   
  \includegraphics[width=1\textwidth]{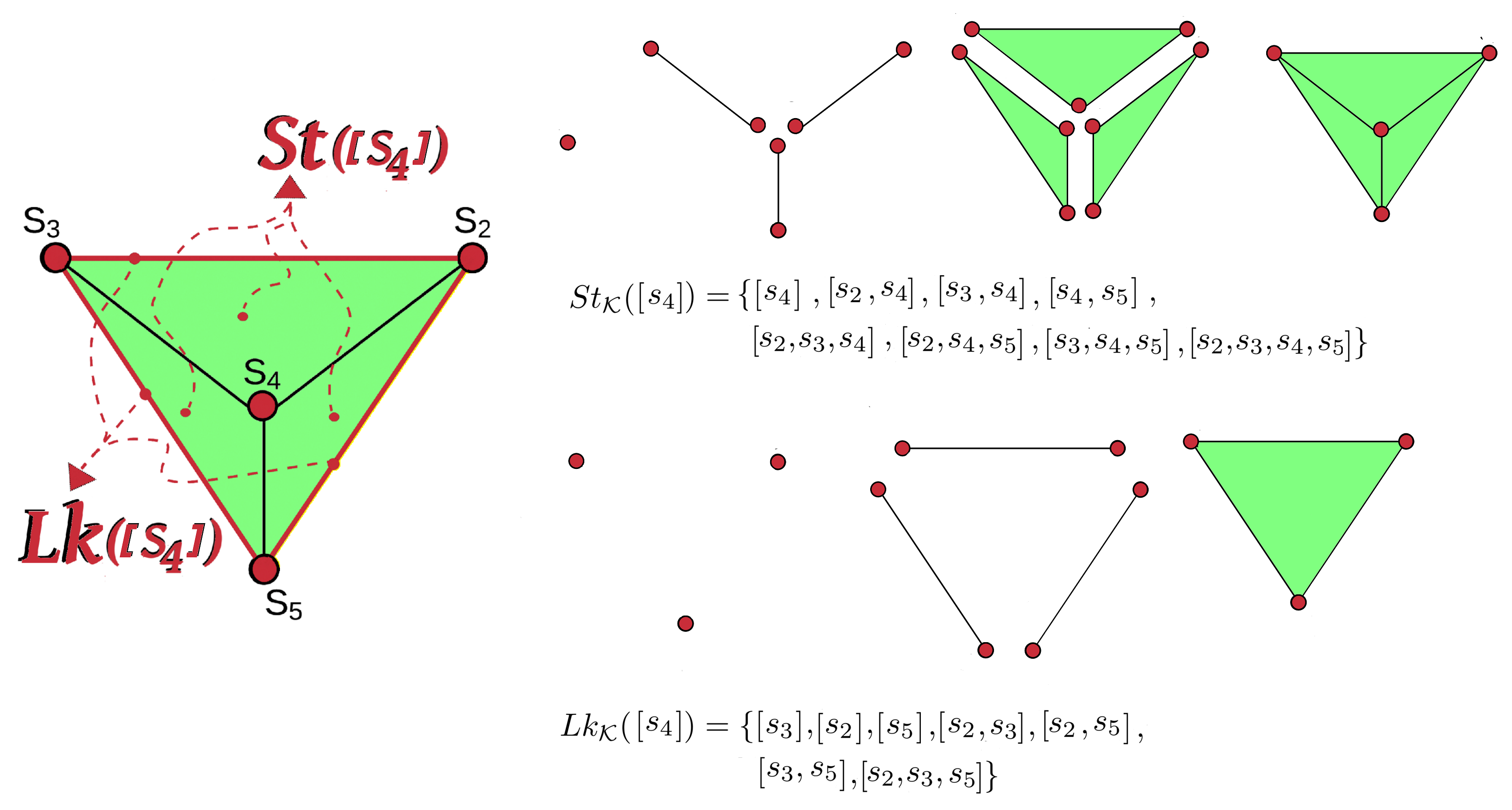}
 \caption{Example of $St_\mathcal{K}([S_4])$ and $Lk_\mathcal{K}([S_4])$ on a given simplicial complex $\mathcal{K}$.}
 \label{fig:linkstar}
\end{figure}

\subsection{Persistent Homology}

Persistent homology is a tool to find topological features in a metric space~\cite{EdelsbrunnerHarer2010, Carlsson:Bulletin}. As a general rule, the objective of persistent homology is to track how topological features on a topological space appear and disappear when a scale value (usually a radius) varies incrementally, in a process known as filtration~\cite{EDelsbrunnerMorozov2014, zomorodian_2005}.
\begin{definition}\textbf{(Sub-complex):}\label{def:subcomplex}
Let $\mathcal{K}$ be a simplicial complex. $\mathcal{K}'$ is a sub-complex of $\mathcal{K}$ if $\mathcal{K}' \subseteq \mathcal{K}$ and $\mathcal{K}'$ is also a simplicial complex.
\end{definition}
\begin{definition}\textbf{(Filtration):}\label{def:filtration} Let $\mathcal{K}$ be a simplicial complex.  A filtration $\mathcal{F}$ on $\mathcal{K}$ is a succession of increasing sub-complexes of $\mathcal{K}$:

 \[
 \emptyset \subseteq \mathcal{K}_0 \subseteq \mathcal{K}_1 \subseteq \mathcal{K}_2 \subseteq \mathcal{K}_3 \subseteq \cdots \subseteq \mathcal{K}_n = \mathcal{K}\]
 In this case, $\mathcal{K}$ is called a filtered simplicial complex.
 \end{definition}

In most of simplicial complexes where the simplices are determined by proximity under a distance function (as in the case of \v{C}ech or VR complexes), a filtration on a simplicial complex $\mathcal{K}$ is obtained by taking a sequence of positive values $0<\varepsilon_{0}\leq \varepsilon_{1}\leq\dots\leq\varepsilon_{n}$, and the complex $\mathcal{K}_{i}$ corresponds to the value $\varepsilon_{i}$.

\begin{definition}[Filtration level function, $\psi$]\label{def:get_level_function}
Let $\mathcal{K}$ be a finite simplicial complex and $\mathcal{F}$ a filtration on $\mathcal{K}$: $\emptyset \subseteq \mathcal{K}_0 \subseteq \mathcal{K}_1 \subseteq \mathcal{K}_2 \subseteq \mathcal{K}_3 \subseteq \cdots \subseteq \mathcal{K}_n = \mathcal{K}$. The filtration level function $\psi_ \mathcal{F}$ is defined on $\{\varepsilon_0,\varepsilon_1,\dots,\varepsilon_n\} $  by $\psi_\mathcal{F}(\varepsilon_i)= \mathcal{K}_i$.  
\end{definition}

 A filtration could be understood as a method to build the whole simplicial complex $\mathcal{K}$ from a ``family'' of sub-complexes incrementally sorted according to some criteria, where each level $i$ corresponds to the ``birth'' or ``death'' of a topological feature as described in Definition~\ref{def:birth}.  

 This process is illustrated in Figure~\ref{fig:filtracion}.
 
 \begin{figure}[H] 
    \centering
    \includegraphics[width=0.7\textwidth]{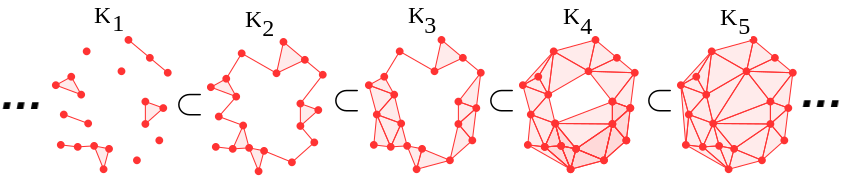}
    \caption{A fragment of a simplicial complex filtration.}\label{fig:filtracion}
\end{figure}

 \begin{definition} [Filtration value collection $\mathcal{E}_\mathcal{K}$]\label{def:fvalues} Let $\mathcal{K}$ be a filtered simplicial complex and $\mathcal{F}$ a filtration on $\mathcal{K}$. Let $\mathcal{E}_\mathcal{K} = \{\varepsilon_0, \varepsilon_1, \dots, \varepsilon_n\} \subset \mathbb{R}$ be a set of non-negative numbers such that $0\leq\varepsilon_0 < \varepsilon_1 < \cdots < \varepsilon_n$, where $\varepsilon_i \in \mathcal{E}_\mathcal{K}$ is a filtration value (radius) applied to build the sub-complex $\mathcal{K}_i \subseteq \mathcal{K}$ in the filtration $\mathcal{F}$. The set $\mathcal{E}_\mathcal{K}$ is called the filtration value collection associated to $\mathcal{F}$. 
 \end{definition}
 
 \begin{definition} [Filtration value of a q-simplex $\xi_\mathcal{K}(\sigma)$]\label{def:getfvalue} Let $\mathcal{K}$ be a filtered simplicial complex and $\mathcal{E}_\mathcal{K}$ its filtration value collection. Let $\sigma\in \mathcal{K}$ be a q-simplex. 
 If $\sigma \in \mathcal{K}_j$ but $\sigma \not \in \mathcal{K}_i, \forall i<j$, then $ \xi_\mathcal{K}(\sigma)=\varepsilon_{j}$ is the filtration value of $\sigma$. \par
 \end{definition}
 
 Note that $\tau \leq \sigma \implies \xi_\mathcal{K}(\tau) \leq \xi_\mathcal{K}(\sigma)$, which means that in a filtered simplicial complex $\mathcal{K}$, every simplex $\tau \in \mathcal{K}$ appears before all its co-faces $\sigma \in \mathcal{K}$.
 
 \begin{definition}[Birth and Death]\label{def:birth}
 \textit{Birth} is a concept to describe the filtration level when a new topological feature appears. Similarly, \textit{death} refers to the filtration level when a topological features disappears. Thus, a persistence interval (birth, death) is the ``lifetime'' of a given topological feature~\cite{EdelsbrunnerHarer2010}. 
  \end{definition}

\section{Proposed Classification Method}\label{scc:classifier}

Let $P \subset \mathbb{R}^{n}$ be a finite point set, where $\mathbb{R}^{n}$ is a feature space. Suppose $P$ is divided in two subspaces $P = S \cup X$, where $S$ is the training set, and $X$ is the test set. Let $L=\{l_{1},l_{2},\dots,l_{N}\}$ be the label set, and let $\mathbb{T} = \{(p, l): p \in P, l \in L\}$ be the association space, which relates every point $p \in P$ with a unique label $l \in L$. Let $T_S$ and $T_X$ be the two disjoint association sets corresponding to $S$ and $X$, respectively, where $\mathbb{T} = T_S \cup T_X$. In this setting, the real label list, $Y = \{l_i \mid (x_i, l_i) \in T_X\}$, is the list of labels assigned to each element of $X$ in the association set $T_X$. 
Thus, the classification problem can be defined as how to predict a suitable label $l \in L$ for every $x \in X$ by assuming the association set $T_X$ as unknown. Consequently, the predicted label list, $\hat{Y} \subset L^{\abs{T_X}}$, will be the resulting collection of labels after classifying each $x \in X$. Since $\abs{Y} = \abs{\hat{Y}}$, it is common to use $Y$ to evaluate the quality of $\hat{Y}$. Depending on the size of $X$ and $S$ the problem is known as supervised classification ($\abs{X} \leq \abs{S}$), or semi-supervised classification ($\abs{S} < \abs{X}$). \par
A classification method based on TDA is presented in this section. Overall, a filtered simplicial complex $\mathcal{K}$ is built over $S \cup X$ to generate data relationships.~Only a few of those relationships will be relevant relationships between data points. In this context, a relationship between points is real if this relationship is part of the data's hidden structure. Thus, persistent homology is applied to capture the real structure of the dataset. This information is helpful to detect a subset of relationships likely to be real. The proposed method is based on the supposition that on the filtration a sub-complex $\mathcal{K}_i\subset \mathcal{K}$ exists, whose simplices represent real data relationships. For every q-simplex $\sigma\in\mathcal{K}_i$, the set of vertices  of $\sigma$, $\mathcal{V}(\sigma)$, will be split into labeled and unlabeled points, where any of these subsets could be empty. The fact that a point set $\{v_0, v_1,\dots, v_q\} \subset S \cup X$ belongs to a q-simplex $\sigma \in \mathcal{K}$ implies a similarity or dissimilarity relationship between points $v_0, v_1,\dots, v_q$. This implicit relationship among data is applied to propagate labels from labeled points to unlabeled points. Thus, a link-based labeling propagation method is developed to make a suitable label prediction for each point $x \in X$. \par

\subsection{Link-based label propagation function}

On a filtered simplicial complex $\mathcal{K}$, the neighborhood relationships of a q-simplex $\sigma \in \mathcal{K}$ could be recovered by using the link, star, and closed star concepts (Definition~\ref{def:starlink}). A key component of the proposed method is the label propagation over a filtered simplicial complex. Given a simplicial complex $\mathcal{K}$, a separation between useful simplices, and not-useful simplices (see Definition~\ref{def:useful-simplex}) needs to be considered. This simplicial classification is helpful because useful-simplices contribute to discriminate more labeling information during the propagation and label assignment process. In this way, those sub-complexes $\mathcal{K}_i \in \mathcal{K}$ with an appropriate distribution of useful-simplices will be preferred. 

\begin{definition}\label{def:useful-simplex}
\textbf{(useful-simplex, and non-useful-simplex)}

    Let $\mathcal{K}$ be a simplicial complex built on $S \cup X$, and $\sigma \in \mathcal{K}$ be a q-simplex. We say $\sigma$ is a useful-simplex if it contains more elements from $S$ than elements from $X$. In another case, then $\sigma$ is a non-useful-simplex.  
\end{definition}

\subsubsection{The labeling function}~\label{sec:assoc_fun}

Let $\mathcal{K}$ be a finite simplicial complex built on $S\cup X$, and $\mathcal{F}$ be a filtration on $\mathcal{K}$: $\emptyset \subseteq \mathcal{K}_0 \subseteq \mathcal{K}_1 \subseteq \mathcal{K}_2  \subseteq \cdots \subseteq \mathcal{K}_n = \mathcal{K}$. Suppose a preferred subcomplex $\mathcal{K}_{i}$ in the filtration $\mathcal{F}$ has been selected. Let $A$ be the $\mathbb{R}$-module with generators $\hat{l}_1, \hat{l}_2,\dots, \hat{l}_N$. The generator $\hat{l}_j$ will be associated to the label $l_{j}$ according to the following definition:

\begin{definition}[Association function]\label{def:assoc_funct} Let  $\Phi_{i}: \mathcal{K}_{i}\rightarrow A$ be the association function defined on a 0-simplex $v\in \mathcal{K}_{i}$ as  $\Phi_{i}(v) =\hat{l}_{j}$ if $(v,l_{j})\in T_S$ and $\Phi_{i}(v)=0\in A$ in any other case. The association function can be extended to a q-simplex $\sigma$ as $\Phi_{i}(\sigma)= \sum_{v\in \mathcal{V}(\sigma)}\Phi_{i}(v)$.
\end{definition}

As an intermediate step to propagate labels from labeled points in $S$ to the unlabeled points in $X$ by means of the link operation in simplicial complexes, define the extension function  in $X$ as follows:

\begin{definition}[Extension function]\label{def:assoc_func_psi} Let $\Psi_{i}: X\rightarrow A$ be the function defined on a point $v\in X$ by
\begin{equation}\label{eq:extension}
 \Psi_{i}(v)=\sum_{\sigma\in Lk_{\mathcal{K}_{i}}([v])}\frac{1}{\xi_\mathcal{K}(\sigma)} \cdot \Phi_{i}(\sigma)
\end{equation}
\end{definition}

     In Equation~\ref{eq:extension}, for every q-simplex $\sigma \in Lk_{\mathcal{K}_{i}}([v])$, the filtration value $\xi_\mathcal{K}(\sigma)$ is applied to prioritize the influence of $\sigma$ to label $v$. Let $\alpha, \beta \in Lk_{\mathcal{K}_{i}}([v])$ be two simplices, such that $\xi_\mathcal{K}(\alpha) < \xi_\mathcal{K}(\beta)$. This condition implies that the vertices of $\alpha$ cluster around $v$ earlier than the vertices of $\beta$ do since they were added first to the filtration. In consequence, $\alpha$ contributions should be more important than $\beta$ contributions. \par

According to the previous definitions, given a point $v\in X$, the evaluation of the extension function at $v$ would be $\Psi_{i} (v)=\sum_{j=1}^{N} a_{j}\cdot \hat{l}_{j}$, where $a_{j}\in \mathbb{R}^{+}\cup\{0\}, j=1,\dots, N$.

 \begin{definition}[Labeling function]\label{def:lab_func}  Let $v$ be a point in $X$ such that  $\Psi_{i} (v)=\sum_{j=1}^{N} a_{j}\cdot \hat{l}_{j}$. If $\tilde{a}$ is the maximum value in $\{a_{j}\}_{j=1}^{N}$, define the labeling function $\Upsilon_i$ at $v$ as  $\Upsilon_i(v) =  l_k$ where $k$ is uniformly selected from the set $\{j\mid a_{j}=\tilde{a}\}$.
 
 \end{definition}

If there exists a unique maximum in the set $\{a_{j}\}_{j=1}^{N}$ from the previous definition, then the labeling function is uniquely defined at $v$. In most of datasets where the proposed TDA classification method was tested, the label assignment of each point in $X$ was uniquely defined. Figure~\ref{fig:epsilon} shows the labeling process on a previously selected sub-complex.

\begin{figure}[H]
    \centering
    \includegraphics[width=1\linewidth]{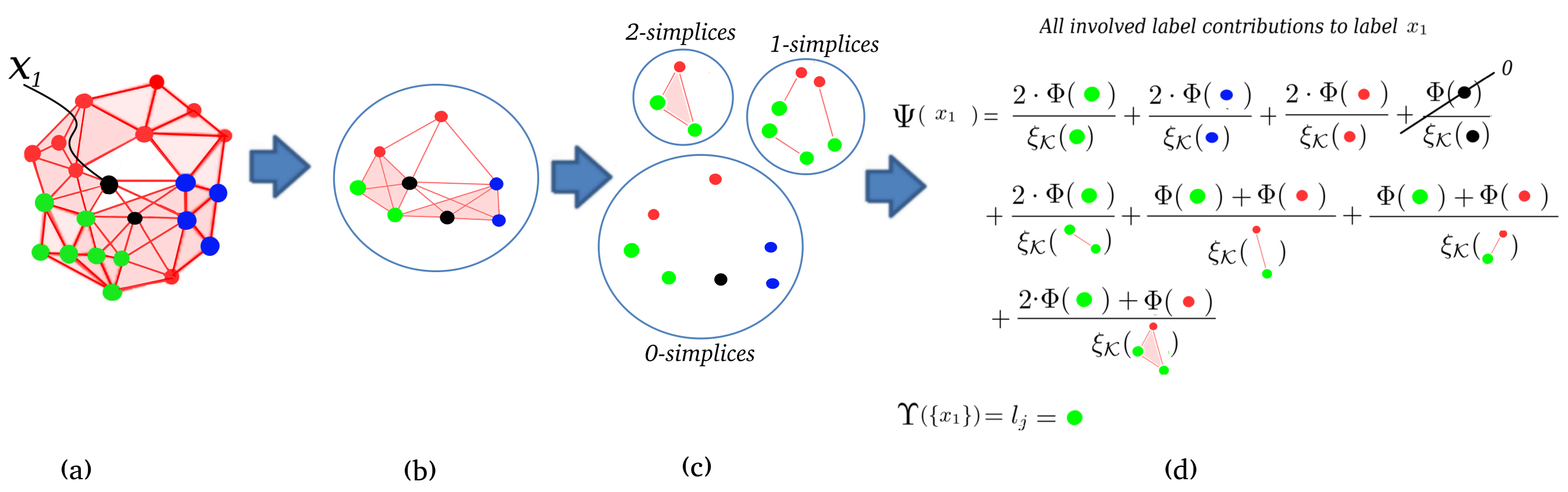}
    \caption{Example of execution of the labeling function on unlabeled points $x_1$. In \textbf{(a)}, a sub-complex $\mathcal{K}_i \subseteq \mathcal{K}$ is presented where the labeling will be executed. The neighborhood of $x_1$ is drawn in \textbf{(b)}. In \textbf{(c)},  $Lk([x_1])$ is shown divided into 0-simplices, 1-simplices, and 2-simplices. In \textbf{(d)}, the extension function $\Psi(x_1)$ is executed, and finally the labeling function $\Upsilon(x_1)$ assigns a green label to $x_1$. The $\xi_\mathcal{K}$ function disambiguates label contributions, which at the beginning seems to be a tie between labels.}\label{fig:epsilon}
\end{figure}
    
\subsection{Classification by using simplicial complexes and persistent homology}

The proposed method computes the predicted label list $\hat{Y}$ corresponding to $X$. The entire process is summarized in Algorithm~\ref{alg:TDABC}. Following subsections explain each step in detail. \par
\begin{algorithm}[H]
\begin{algorithmic}[1]
\caption{\textbf{TDABC:} TDA-Based Classification Algorithm}\label{alg:TDABC}

\REQUIRE A training set $S \neq \emptyset$. \newline A test set $X \neq \emptyset$ to be classified.\newline The incomplete association set $T'$.
\ENSURE A prediction list $\hat{Y}$ of $X$ by using $T'$.
\STATE A filtered simplicial complex $\mathcal{K}$ is constructed by using the Algorithm~\ref{alg:direct-vr}\label{step:filtration}.
\STATE Obtain the prediction list $\hat{Y} = (l_1, l_2, \cdots, l_{\abs{X}})$, where each $l_i \in \hat{Y}$ is the most reliable label corresponding to $x_i \in X$, with $1 \leq i \leq \abs{X}$ by using $\mathcal{K}$ (Algorithm~\ref{alg:labeling})\label{step:label}
\RETURN the prediction list $\hat{Y}$. 
\end{algorithmic}
\end{algorithm}

\subsubsection{Building the filtered simplicial complex.}\label{scc:construction}

Let $S$ and $X$ be two point sets, where $S$ is the labeled set, and $X$ is the unlabeled set. The filtered simplicial complex $\mathcal{K}$ is built on $P = S \cup X$. A maximal dimension $2 \leq q \ll \abs{P}$ is given to control the simplicial complex combinatorial growing. Algorithm~\ref{alg:direct-vr} illustrates this process. 


\begin{algorithm} [H]
\caption{Construction of a filtered simplicial complex $\mathcal{K}$}\label{alg:direct-vr}
\begin{algorithmic}[1]
\REQUIRE A non-empty training set $S$. A non-empty test set $X$. An $0 < \epsilon \leq 1$ increment value to change the filtration level. Let $n$ be the maximum desired level of the filtration on $\mathcal{K}$.  
\ENSURE a filtered simplicial complex $\mathcal{K}$.
\STATE Unify $S$ and $X$ by $P \leftarrow S \cup X$  
\STATE $i \leftarrow 0$ 
\STATE $\varepsilon_i \leftarrow 0$
\STATE $\mathcal{K} \leftarrow \{\}$
\WHILE {$i \leq n$} 
    \STATE Build a simplicial complex $\mathcal{K}_i$ from $P$ by using an $\varepsilon_i$ value, with $dim(\mathcal{K}_i) \leq q$, according to~\cite{EdelsbrunnerHarer2010, DBLP:journals/talg/BoissonnatS18, DBLP:journals/algorithmica/BoissonnatM14, DBLP:journals/algorithmica/BoissonnatST17}
    
    \IF {$\mathcal{K} \neq \mathcal{K}_i$}
        \STATE $\mathcal{K} \leftarrow \mathcal{K} \cup \mathcal{K}_i$    
        \STATE $i \leftarrow i + 1$
        \STATE $\varepsilon_i \leftarrow \varepsilon_{i-1} + \epsilon$
    \ELSE
        \STATE $i \leftarrow n$
    \ENDIF
\ENDWHILE
\RETURN $\mathcal{K}$
\end{algorithmic}
\end{algorithm}

\subsubsection{Obtaining the most reliable label}\label{scc:labeling}

Once a filtered simplicial complex $\mathcal{K}$ is obtained, all the elements of~$X$ need to be labeled. Because of the full connectivity on $\mathcal{K}$, it is highly likely that all the vertices are interconnected. Then $dim(\mathcal{K})=N-1$, where $N = \abs{P}$, and the number of simplices in $\mathcal{K}$ would be $2^{N}$. \par
Therefore, if functions $\Psi_{n}(x)$ and $\Upsilon_{n}(x)$ from Definition~\ref{def:assoc_func_psi} and \ref{def:lab_func} are computed using $\mathcal{K}_n=\mathcal{K}$, it could occur that a given point $x \in X $  has contributions of all possible labels, due to the high number of co-faces of the analyzed simplices. In this context, it is worth understanding which simplices of $\mathcal{K}$ containing some $x \in X$ are reliable or which are noise to perform the label propagation and labeling.
Then, the purpose is to choose a sub-complex $\mathcal{K}_i$ from the filtered simplicial complex $\mathcal{K}$, such that: i) $\mathcal{K}_i$ constitutes a good approximation to the real structure of the dataset, and ii) there are enough useful simplices in $\mathcal{K}_i$ to label each point $x \in X$. In this vein, persistent homology is used to guide the process of selecting $\mathcal{K}_i$, reduce the classification space, and guarantee the useful simplices inside the selected sub-complex. \par  
With persistence homology, multi-dimensional topological features are detected. For a filtered simplicial complex $\mathcal{K}$ of dimension $q$, persistent homology will compute up to $q$-dimensional homology groups. Each topological feature represented by an element in a  homology group of a given dimension will be represented by a persistence interval $(birth, death)\in\mathbb{R}^2$ (see Definition~\ref{def:birth}).\par

The collection of simplices that shapes one topological feature belongs to a homology class. However, it is no trivial to recover information about every simplex belonging to that homology class, except by the simplex generator of that topological feature. For example, a 0-simplex connects to another 0-simplex and creates a 1-simplex. The 1-simplex, which joins together two 1-simplices, creates a hole. A 2-simplex is attached to the other three 2-simplices and becomes a 3-simplex, leaving a void inside. Precisely, the moment when this happens is the birth of a topological feature and the death of another topological feature of a lower dimension. As it can be noticed, the relation between simplices and topological features is injective, and it is only related to the generator. Thus, only the generator of a topological feature might be recovered from a persistence interval at its birth.\par

Based on the assumption that any simplex is related to only one persistence interval, this relation will persist from its birth to its death. Therefore, all persistence intervals which have intercepted their lifetimes will have their corresponding related simplices coexisting during the interception time. Eventually, when a homology class dies, their connection with simplices also dies, and they are no longer considered, at least not directly. A conventional way to get a well-defined membership relation from simplices to a homology class is to look at the simplices associated with the birth of a homology class, which are the only simplices that can reliably be associated with the homology class. \par


The challenge is to find appropriate homology classes to ask for their associated simplices. It is known that long life invariants (high $death-birth$) represent topological features, while short life invariants are commonly considered noise. However, short life persistence intervals could also mean local topological features or high dimensional topological features, which could be more profitable in useful-simplices. On this scenario, persistence intervals will be considered from higher homology groups ($q-1$) downwards (see Algorithm~\ref{alg:piD}). \par
Persistent homology is then considered to recover topological features which represent meaningful data relationships. Some of those topological features hopefully will represent (in their birth) a level of filtration $\mathcal{K}_i$ that maximizes the number of useful-simplices associated with each $x \in X$. As a result, it is highly likely that we obtain a reliable label  for every $ x \in X$. \par
Persistent homology is computed according to~\cite{EdelsbrunnerHarer2010, EDelsbrunnerMorozov2014,  DBLP:journals/dcg/SilvaMV11,Dey2014ComputingTP}, and a collection $\mathbb{D} = \{D^0, D^1, \cdots, D^q\}$ is obtained with $D^i$ the persistence interval set of the $i^{th}$-dimensional homology group. Algorithm~\ref{alg:piD} computes a persistence interval set $D \in \mathbb{D}$, which is the non-empty homology group with higher dimension. 

\begin{algorithm}[H]

\begin{algorithmic}[1]
\caption{\textbf{GetPersistenceIntervalSet:} Computing the persistence interval set $D$}\label{alg:piD}
\REQUIRE A filtered simplicial complex $\mathcal{K}$.
\ENSURE A persistence interval set $D$, where:\newline
$D \leftarrow \{d_i \mid d_i=(birth, death)\}$.
\STATE $\mathbb{D} \leftarrow Compute Persistent Homology(\mathcal{K})$~\cite{DBLP:journals/dcg/SilvaMV11, EDelsbrunnerMorozov2014} with $\mathbb{D} = \{D^0, D^1, \cdots, D^q\}, D^i$ the persistence interval set of the $i^{th}$-dimensional homology groups. 
\STATE $D \leftarrow \{\}$
\STATE $i \leftarrow q$
\WHILE{$D == \emptyset$ and $i \geq 1$}
\STATE $D \leftarrow D^i$ with $D^i \in \mathbb{D}$
\STATE $i \leftarrow i-1$
\ENDWHILE
\RETURN $D$
\end{algorithmic}
\end{algorithm}

Let $d \in D$ be a persistence interval. Then $life(d) = d[death]-d[birth]$. We notice that $life(d)$ becomes undefined for immortal topological feature (i.e. infinite death time). To overcome this issue, it is enough to change the death time from infinite to the maximum of the filtration value collection $\mathcal{E}_\mathcal{K}$. We call this a $\vartheta$-transformation (see Equation~\ref{eq:delta}). Thus, a new function $int(d) = life(\vartheta(d))$ is defined to apply the $\vartheta$-transformation before life(d) is called. 

\begin{equation}\label{eq:delta}
    \vartheta (d) = \left\{\begin{matrix}
d & \mbox{ iff } & d[death] \neq \infty,\\ 
(d[birth], \max (\mathcal{E}_\mathcal{K})) & \mbox{ iff } & d[death] = \infty.
\end{matrix}\right.
\end{equation}

A desired persistence interval $d \in D$ is selected by using the (naive) functions defined on Equation~\ref{eq:maxint}, Equation~\ref{eq:randint}, and Equation~\ref{eq:avgint}: \par
\begin{enumerate}[(a)]
    \item The maximum persistence interval: 
    \begin{equation}\label{eq:maxint}
      d_m = MaxInt(D) = \argmax_{d \in D}(int(d)).
    \end{equation}
    \item A persistence interval selected in a random way: 
    \begin{equation}\label{eq:randint}
      d_r = RandInt(D) = \vartheta(random(D)).
    \end{equation}
    \item The closest interval to the persistence intervals average: 
    \begin{equation}\label{eq:avgint}
      d_a = AvgInt(D) = \argmin_{d \in D} \left | int(d) - avg(D) \right |,
    \end{equation}
    where $avg(D) = \frac{1}{\abs{D}} \cdot \sum_{d_i \in D} int(d_i).$
\end{enumerate}

Although homology groups are different, the birth and death of persistence intervals values are absolutes. Even when those persistence intervals that contain high dimensional invariants were selected, low dimensional topological features are not necessarily excluded. In Algorithm~\ref{alg:labeling}, a persistence interval $d$ is selected from a filtration, to recover a sub-complex $\mathcal{K}_{i} \subset \mathcal{K}$ and classify all $x \in X$. \par
Because of the injectivity between simplices and birth time of persistence intervals, the sub-complex $\mathcal{K}_{d[birth]} \subset \mathcal{K}$ might be selected. Nevertheless, the sub-complexes $\mathcal{K}_{\frac{d[birth]+d[death]}{2}} \subset \mathcal{K}$ and $\mathcal{K}_{d[death]} \subset \mathcal{K}$ could be selected as well. In these cases, the middle time ($\frac{d[birth]+d[death]}{2}$) and death time of persistence interval $d$ will capture all those simplices which are generators of topological features still alive (or born) on the middle and death times. The choice between birth, middle, or death time to select the most appropriate sub-complex seems related to the homology group dimension. When the selected homology group has a high dimension, $\mathcal{K}_{d[birth]}$ gives good precision on classification. On the other hand, if a 0-dimensional homology group is selected, the sub-complex $\mathcal{K}_{d[death]}$  should be the best choice. The sub-complex $\mathcal{K}_{\frac{d[birth]+d[death]}{2}}$ could be also selected on 1-dimensional homology groups. As a result, the birth time was chosen (see Algorithm~\ref{alg:labeling}) to select the sub-complex because it is always guaranteed to be present.\par

Figure~\ref{fig:circlepinfo} shows the selection of a sub-complex $\mathcal{K}_i \subseteq \mathcal{K}$,  where $\mathcal{K}$ is a filtered simplicial complex built on the Circles dataset (with noise = 10), one of the artificial datasets used to evaluate the proposed method in Section~\ref{scc:results}. The selection of $\mathcal{K}_i$ is guided by the persistent homology information and the application of $MaxInt(\cdot)$, and $RandInt(\cdot)$ selection functions (see Equation~\ref{eq:maxint}, and Equation~\ref{eq:randint}) to select an appropriate persistence interval according to each criterion.  Note that $RandInt(\cdot)$ was coincident with $AvgInt(\cdot)$, and the results of only one persistence interval were shown in Figure~\ref{fig:circlepinfo}. 

\begin{figure}[H]
\begin{center}
\includegraphics[width=1\textwidth]{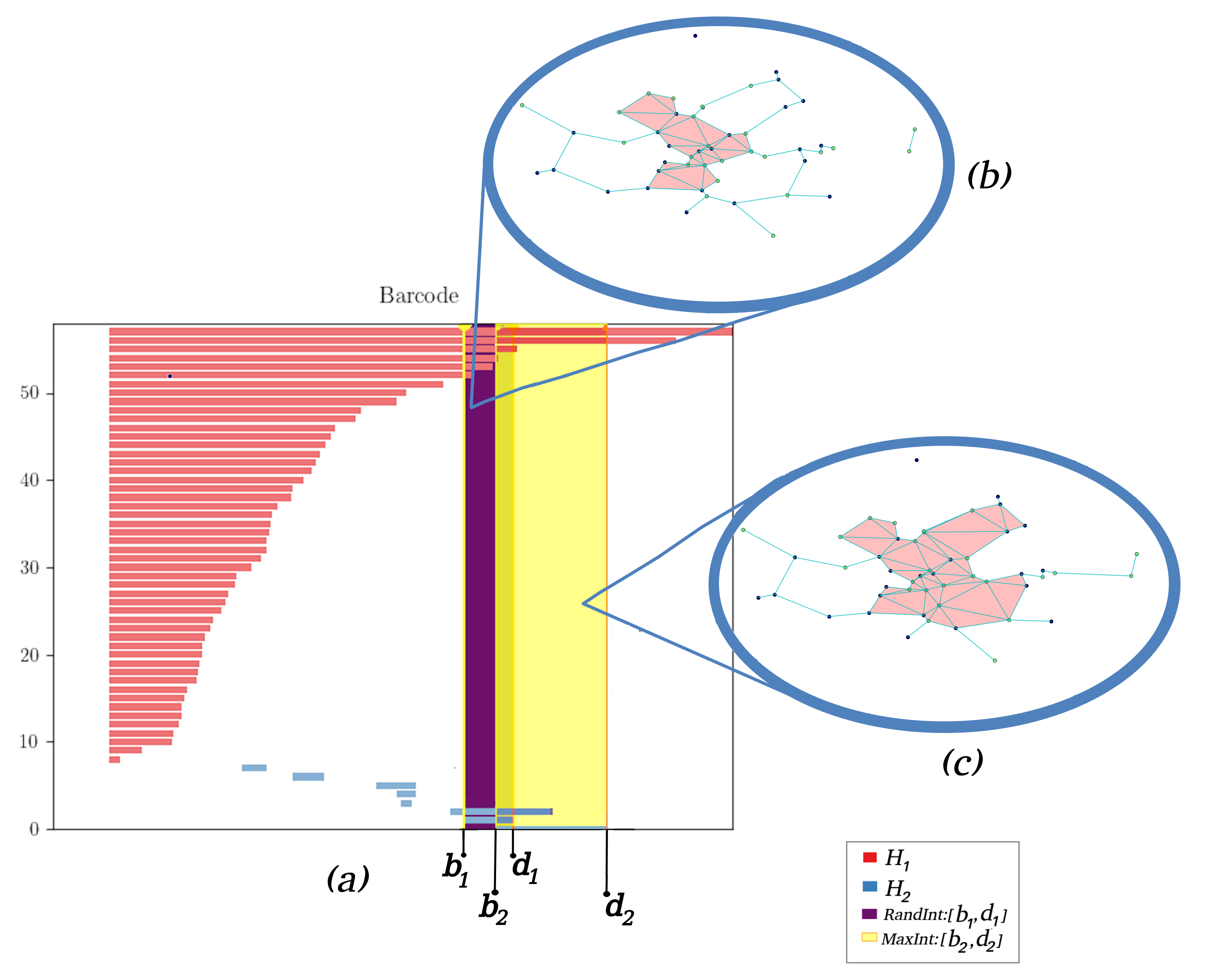}
\caption{A filtered simplicial complex $\mathcal{K}$ was built on the Circles dataset, detailed in Section~\ref{scc:artificial}.~Persistent homology was computed to guide the sub-complex selection. A barcode representation of persistent homology is shown in (a), with two selected persistence intervals $[b_1, d_1]$ (purple), and $[b_2, d_2]$ (yellow) corresponding to the $RandInt(\cdot)$; and $MaxInt(\cdot)$ selection functions respectively. In (b), the sub-complex $\mathcal{K}_{d_1} \subseteq \mathcal{K}$, and $\mathcal{K}_{d_2} \subseteq \mathcal{K}$ is shown in (c).}\label{fig:circlepinfo}    
\end{center}
\end{figure}

\begin{algorithm}[H]
\begin{algorithmic}[1]
\caption{\textbf{Labeling:} Labeling a test set $X$.}\label{alg:labeling}
\REQUIRE A filtered simplicial complex $\mathcal{K}$ by a filtration $\mathcal{F}$. \newline A non-empty test set $X$. 
\ENSURE A predicted labels list $\hat{Y}$ of $X$.
\STATE $D \leftarrow GetPersistenceIntervalSet(\mathcal{K})$ where:\newline
$D = \{d_i \mid d_i=(birth, death)\}$
\STATE Get a desired persistence interval $d$ where: \newline $d \in \{MaxInt(D), RandInt(D), AvgInt(D)\}$\label{step: selection} 
\STATE $\varepsilon_i \leftarrow d[birth]$ 
\STATE $\mathcal{K}_{i} \leftarrow \psi_\mathcal{F}(\varepsilon_i)$ see Definition~\ref{def:get_level_function}\label{step:extrange} 
\STATE $\hat{Y} \leftarrow \{\}$
\WHILE{$X \neq \emptyset$}
\STATE $x \in X$
\STATE $l \leftarrow \Upsilon_{{i}}(x)$ see Definition~\ref{def:assoc_func_psi}
\STATE $\hat{Y} \leftarrow \hat{Y} \cup \{l\}$
\STATE $X \leftarrow X \setminus \{x\}$
\ENDWHILE
\RETURN $\hat{Y}$
\end{algorithmic}
\end{algorithm}

An overview of the proposed method is presented in Figure~\ref{fig:overall}. To classify two black points $x_1, x_2 \in X$, a 4-steps process is executed. 

\begin{figure}[H]
    \centering
    \includegraphics[width=1\textwidth]{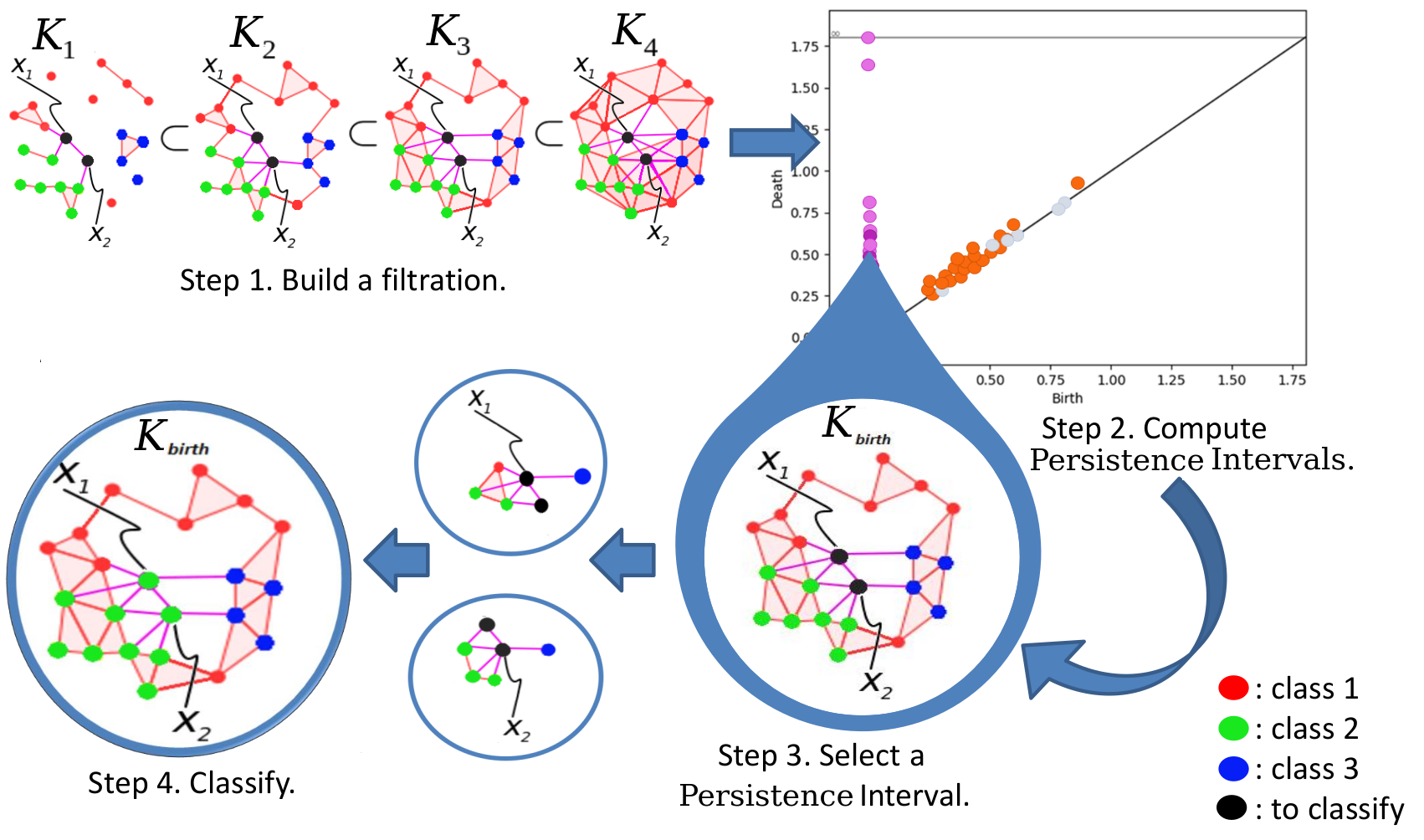}
    \caption{Overall TDABC algorithm (see Algorithm~\ref{alg:TDABC}). The first step is to build a filtered simplicial complex $\mathcal{K}$, following Algorithm~\ref{alg:direct-vr}. On the second step, persistent homology is computed to recover topological features according to~\cite{DBLP:journals/dcg/SilvaMV11} and \cite{Dey2014ComputingTP}. The third step consists of applying Algorithm~\ref{alg:piD} to compute a persistence interval by using Equation~\ref{eq:maxint}, \ref{eq:randint}, and \ref{eq:avgint}. A sub-complex $\mathcal{K}_i$ is then recovered from the filtration by using the birth of the selected interval. Fourth and last, the classification step uses a label propagation approach and takes the most voted label (Algorithm~\ref{alg:labeling} and Definition~\ref{def:lab_func}).} \label{fig:overall}
\end{figure}

\subsection{Implementation}

The proposed TDA-based classifier (TDABC) was implemented on top of the GUDHI library~\cite{gudhi2014, DBLP:journals/algorithmica/BoissonnatST17, DBLP:journals/talg/BoissonnatS18}, which is one of the most complete libraries for building simplicial complexes~\cite{gudhi:FilteredComplexes} and computing homology groups~\cite{DBLP:journals/algorithmica/BoissonnatST17, Otter2017, DBLP:journals/talg/BoissonnatS18, DBLP:journals/algorithmica/BoissonnatM14}. 

\subsubsection{Simplicial complex construction and persistence computation}\label{sec:impl_sc}

 The first step of the proposed TDABC method (see Figure~\ref{fig:overall}) is building a filtered simplicial complex $\mathcal{K}$ on $S \cup X$. For datasets with high dimensions or datasets with too many samples,  the implementation of the Algorithm~\ref{alg:direct-vr} in GUDHI could be impractical due to combinatorial complexity. Consequently, the combinatorial complexity of the simplicial complex must be reduced. To address this problem, the approach followed in this paper is using the edge collapse~\cite{gudhiContraction} method on the GUDHI library. The collapse of edges in GUDHI has to be performed on the 1-skeleton of the simplicial complex and then expand to build all high dimensional simplices up to a maximal dimension $q \ll \abs{S \cup X}$. Algorithm~\ref{alg:gudhidirect} computes a simplex tree by using the edge collapsing method. A collapsing coefficient depends on the maximal dimension $q$, but it could be enhanced by invoking $collapse\_edges$ repeatedly until the simplex tree does not change any more. 

\begin{algorithm}[H]
\begin{algorithmic}[1]
\caption{Build a simplex tree with edge collapsing}\label{alg:gudhidirect}
\REQUIRE A non-empty point set $P = S \cup X$. 
\ENSURE An updated simplex tree $\mathcal{S}$.
\STATE $\mathcal{K} \leftarrow buildSC(P)$, create a simplicial complex on $P$ up to dimension $1$.
\STATE $\mathcal{S} \leftarrow create\_simplex\_tree(\mathcal{K}, max\_dim = 1)$
\STATE $c \leftarrow \left \lceil \frac{q}{3} \right \rceil$ \COMMENT{ a simple coefficient of edge collapsing}
\STATE $\mathcal{S} \leftarrow collapse\_edges(\mathcal{S}, c)$
\STATE $\mathcal{S} \leftarrow expansion(\mathcal{S}, max\_dim=q)$
\RETURN $\mathcal{S}$
\end{algorithmic}
\end{algorithm}

\subsubsection{Persistence computation and persistence interval selection}\label{sec:impl_ph}

In GUDHI, instead of persistent homology, persistent cohomology is computed using the algorithm \cite{DBLP:journals/dcg/SilvaMV11} and \cite{Dey2014ComputingTP} and the Compressed Annotation Matrix data structure implementation presented in \cite{annomatrix}. Due to homology and cohomology's duality both methods compute the same homological information, but cohomology provides richer topological information~\cite{DBLP:journals/dcg/SilvaMV11}. Algorithm~\ref{alg:piD}'s implementation in GUDHI is direct, using the $persistence(\cdot)$ method of the simplex tree data structure.  

\subsubsection{Label propagation implementation}\label{sec:impl_label_propagation}

The extension function $\Psi$ from Definition~\ref{def:assoc_func_psi} depends on the $Lk_{\mathcal{K}}$ operation from Definition~\ref{def:starlink}. Up to now the python interface of GUDHI Library (v.3.3.0)~\cite{gudhi:cython} does not have an implementation of the simplex link operation; however, it provides implementations of star and co-face operators. As a result, a link operation function was implemented based on Equation~\ref{eq:linklemma} from Lemma~\ref{lemma:002}. 

According to Lemma~\ref{lemma:002}, $\Psi$ function from Definition~\ref{def:assoc_func_psi} could be implemented based on $St_\mathcal{K}$. In addition, two ways of removing the $\sigma$ contributions are possible: a strict way and a belated one. The first method is computing strictly $Lk_\mathcal{K}$ using Lemma~\ref{lemma:002}. The advantage of this way is reducing the quantity of invocations of the association function $\Phi(\sigma)$ in Definition~\ref{def:assoc_funct}. In the second way, the $St_\mathcal{K}(\sigma)$ function is used as a whole, and the $\sigma$ contributions are then ignored during function $\Phi(\sigma)$ execution because it would be 0 for unknown 0-simplices. Lemma~\ref{lemma:002} and Definition~\ref{def:assoc_func_psi} show that both approaches are equivalent. \par

In GUDHI, each q-simplex $\sigma \in \mathcal{K}$ represented by a simplex tree $\mathcal{S}$ is related to its filtration value $\xi_\mathcal{K}(\alpha)$. Thus, the $star(\mathcal{S}, \sigma)$ is a function in $\mathcal{S}$ which returns a 2-tuple set $\{(\mu, \xi_\mathcal{K}(\mu)) \mid \mu \in St_\mathcal{K}(\sigma) \}$. This facilitates the implementation of function $\Psi$ and the recovery of the $\varepsilon$-values to impose a priority over simplices and minimize a tie.\par

\section{Results}\label{scc:results}

The proposed TDA-based classifier (TDABC) is sensible to the chosen selection function $RandInt(\cdot), MaxInt(\cdot), \mbox{ and } AvgInt(\cdot)$. Those selection functions can detect a specific sub-complex in the filtered simplicial complex built on the dataset. Due to this dependency, the proposed method's behavior needs to explored by using those functions. Consequently, three versions of TDABC methods are configured to assess the proposed solutions: 

\begin{enumerate}[(i)]
    \item  The TDABC-R classification method using the $RandInt(\cdot)$ selection function.
    \item The TDABC-M, because of the utilization of the $MaxInt(\cdot)$ selection function.
    \item The TDABC-A, which uses the $AvgInt(\cdot)$ selection function.
\end{enumerate}   

\subsection{Selected baseline classifiers }

Three baseline methods were selected to compare the proposed methods:

\begin{enumerate}[(i)]
    \item  The k-Nearest Neighbors (k-NN) implementation from Scikit-Learn~\cite{scikit-learn} was chosen.
    \item The distance-based weighted k-NN from Scikit-Learn was also selected to assess the proposed methods.
\end{enumerate} 

\subsection{Datasets}
 
Several data sets were chosen to evaluate the proposed methods and compare them to the baseline classifiers. Table~\ref{tb:datasets} shows the datasets with some of their characteristics.

\begin{table}[h!]
\centering
\resizebox{\textwidth}{!}{%
\begin{tabular}{|l|l|l|l|l|l|l|l|}
\hline
\textbf{Name} & \textbf{Dimensions} & \textbf{Classes} & \textbf{Size} & \textbf{\begin{tabular}[c]{@{}l@{}}Samples\\ per class\end{tabular}} & \textbf{Noise} & \textbf{Mean} & \textbf{Stdev} \\ \hline
\textbf{Circles} & 2 & 2 & 50 & {[}25,25{]} & 3 & - & - \\ \hline
\textbf{Moon} & 2 & 2 & 200 & {[}100,100{]} & 10 & - & - \\ \hline
\textbf{Swissroll} & 3 & 6 & 300 & {[}50,50,50,50,50,50,{]} & 10 & - & - \\ \hline
\textbf{Normdist} & 350 & 5 & 300 & {[}60,10,50,100,80{]} & - & {[}0,0.3,0.18,0.67,0{]} & 0.486 \\ \hline
\textbf{Sphere} & 3 & 5 & 653 & {[}500,100,25,16,12{]} & - & 0.3 & 0147 \\ \hline
\textbf{Iris} & 4 & 3 & 150 & {[}50,50,50{]} & - & - & - \\ \hline
\textbf{Wine} & 13 & 3 & 178 & {[}59, 71, 48{]} & - & - & - \\ \hline
\textbf{\begin{tabular}[c]{@{}l@{}}Breast \\ Cancer\end{tabular}} & 30 & 2 & 569 & {[}212, 357{]} & - & - & - \\ \hline
\end{tabular}%
}
\caption{Selected datasets to evaluate proposed and baseline classifiers.}
\label{tb:datasets}
\end{table}

 Each dataset will be explained in ths section, and a graphical representation is presented in Figure~\ref{fig:artificial_data} and Figure~\ref{fig:real_data}. There are datasets with more than 3 dimensions. In case datasets involves more than 3 dimensions, a Principal Component Analysis (PCA) was applied to reduce the dimensionality for visualization purposes only. Then, resulting datasets were plotted taking pairwise variables $\binom{3}{2}$ to provide several two-dimensional points of view with the axis XY, XZ, and YZ, respectively. \par

\subsubsection{Artificial datasets}\label{scc:artificial}

A group of datasets was artificially generated:  The Circles, Swissroll, Moon, Normdist, and Sphere datasets (see Figure~\ref{fig:artificial_data}). In this section, details regarding each one of those datasets will be provided.\par
The \textbf{\textit{Circles dataset}} is an artificial and simple dataset that consists of a large circle with a small circle inside. Both circles are Gaussian data with a spherical decision boundary for binary classification. A Gaussian noise factor of 3 was added to the data making the circular boundary more diffused. This dataset was proposed to assess the ability to disentangle or to deal with overlapped data regions. The label set will be $L=\{0, 1\}$ denoting both circles. The point set $P \in \mathbb{R}^2$ will be all samples points from both circles. Figure~\ref{fig:noiseless} shows the Circle dataset without noise. Figure~\ref{fig:artificial_data} presents this dataset with a noise factor of 3. The noisy Circles dataset was selected for experiments.   \par

The \textbf{\textit{Moon dataset}} is a simple dataset generated by making two interleaving half circles. A noise factor of 3 was added to data to make it difficult to separate both half circles. The label set $L = \{0, 1\}$ denoting both classes. The point set $P \in \mathbb{R}^2$ is composed of all generated samples of the dataset. Figure~\ref{fig:noiseless} shows the Moon dataset samples distribution without noise. Figure~\ref{fig:artificial_data} shows this dataset with used noise.\par

The \textbf{\textit{Swissroll dataset}} is an $\mathbb{R}^2$ point set mapped to $\mathbb{R}^3$ with a rolled shape. In this paper, a Swissroll dataset was generated using 300 samples from 5 different classes. Besides, a noise factor of 10 was added to the data, which dissolves the rolled shape almost totally. The label set $L = \{0, 1, 2, 3, 4\}$ will be composed by enumerating all classes. The generated samples will be directly used to build the point set $P \in \mathbb{R}^3$. Figure~\ref{fig:noiseless} shows the Swissroll dataset without noise, and Figure~\ref{fig:artificial_data} shows it with noise.\par

\begin{figure}[H]
\centering
\includegraphics[width=0.8\textwidth]{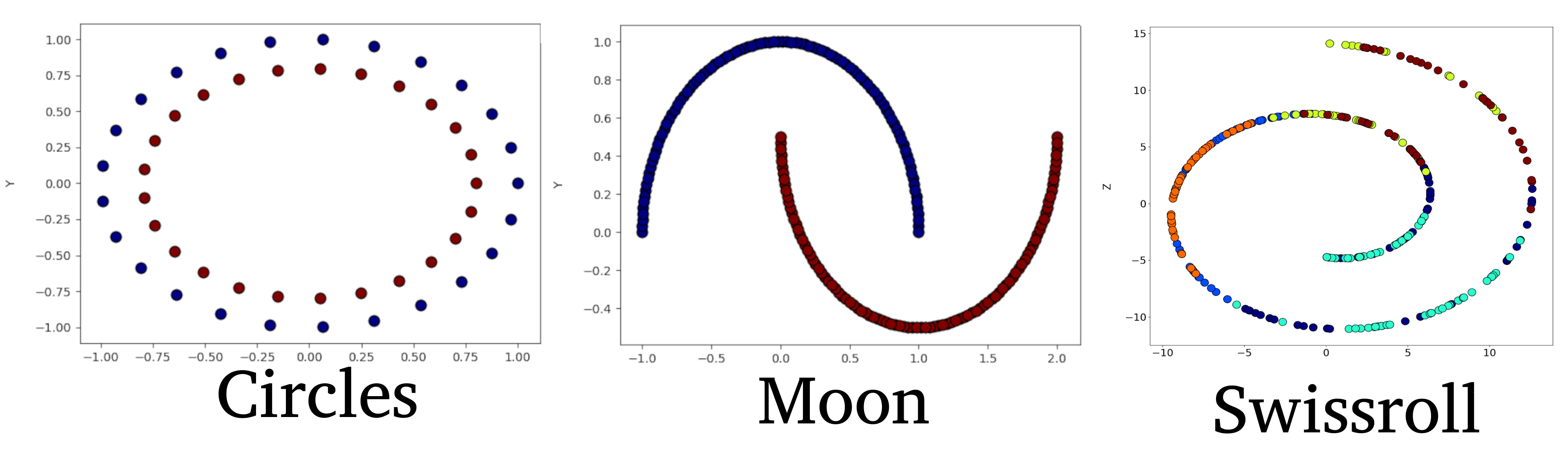}
\caption{The Circles, Moon, and Swissroll artificial datasets without noise. }\label{fig:noiseless}
\end{figure}

An Artificial Dataset Generator framework was implemented for developing datasets distributions. This framework is flexible enough to simulate several complex situations. It is possible to define the desired number of objects (classes) with samples number per class, and global mean and standard deviation or per class. \par
The \textbf{\textit{Normal Distribution based Dataset}} is generated by defining a several per class and overall parameters such as: dataset size, samples dimension, mean per class, standard deviation per class, number of samples per class, number of objects. The number of objects determine the number of classes or labels to be part of the Dataset. The dimension of the Dataset is solved by generating a normal distribution in each component.\par 
An artificial dataset based on mixtures of normal distribution ($Norm\_dist$) was generated using the dataset generation framework. This dataset is a high dimensional point set $P \subset \mathbb{R}^{350}$, with a total size of 300 samples. The label set will be $L = \{0, 1, 2, 3, 4\}$. The point set $P$ is composed by generating a normal distribution across each component. The sample number list is $[60, 10, 50, 100, 80]$. To guide the dispersion and density of the point cloud, we used a mean values collection {\it mean} = $[0, 0.3, 0.18, 0.67, 0]$ and a {\it standard deviation} ($0.486$) per label. Figure~\ref{fig:artificial_data} shows the samples' distribution after the PCA process was applied for visualization.\par

Generating a \textbf{\textit{Sphere-based dataset}} is similar to generating a normal distribution-based one. Although these datasets are always in three dimensions, they are oriented to capture problems associated with data shape, and entanglement between different class samples and diverse class sample distributions and sizes. Figure~\ref{fig:artificial_data} shows a sphere-based dataset $P \subset \mathbb{R}^3$. This data set has a total size of 653 samples, with a label set $L = \{0, 1, 2, 3, 4\}$. The label distribution is also imbalanced with $[500, 100, 25, 16, 12]$. The {\it mean}~($0.397$) and the {\it standard deviation}~($0.147$) are equal per label samples subset.  

\begin{figure}[H]
\centering
\includegraphics[width=0.8\textwidth]{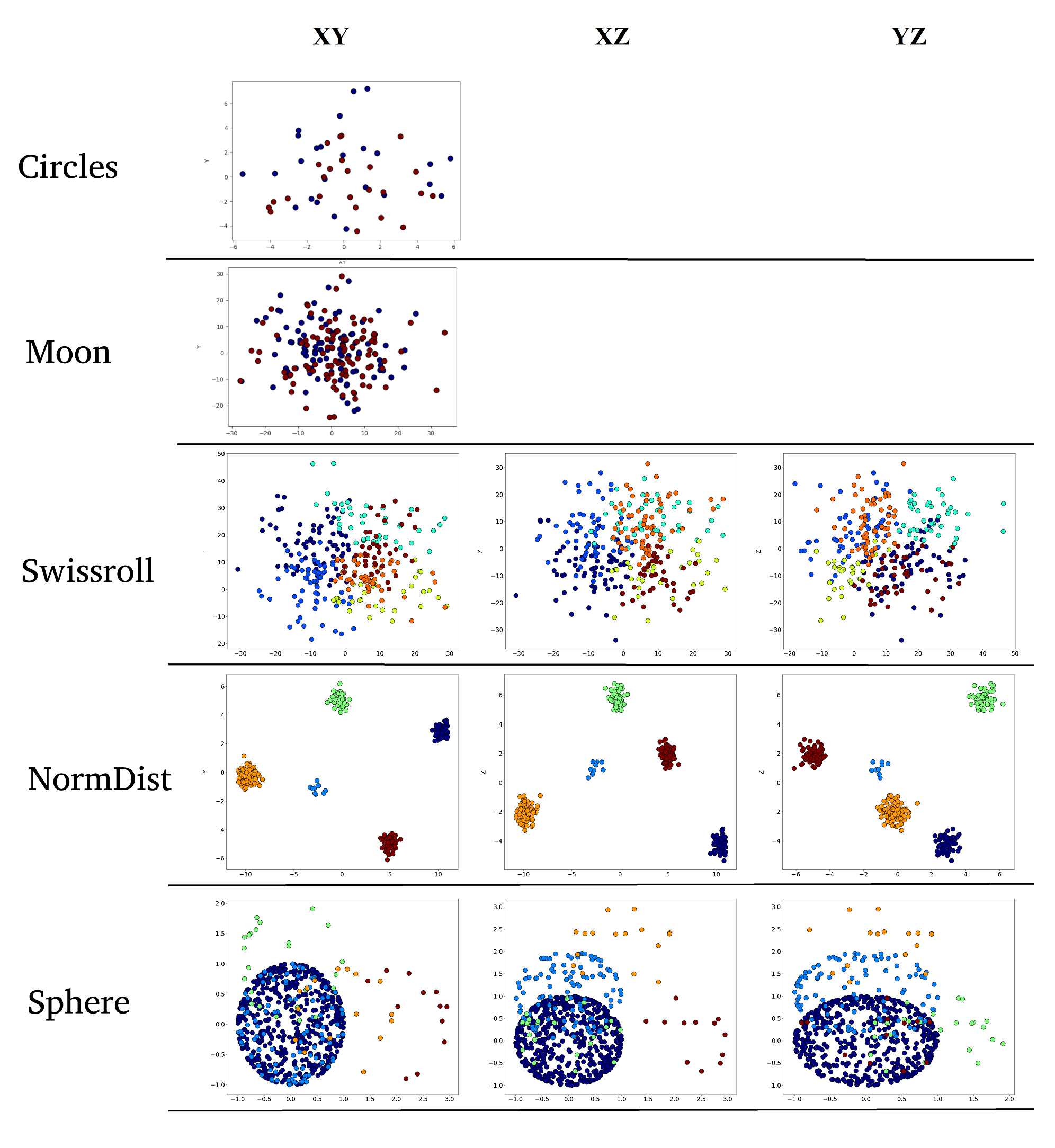}
\caption{The Circles, Moon, Swissroll, Normdist, and Sphere artificial datasets. }\label{fig:artificial_data}
\end{figure}

\subsubsection{Real datasets}
The Iris, Wine, and Breast Cancer datasets were selected as real datasets to compare the proposed classifiers and the baseline ones. In this section, each real dataset will be explained and several of their characteristics will be described.\par
The \textbf{\textit{Iris dataset}}~\cite{Dua:2019} contains 3 classes of 50 instances each, where each class refers to a type of Iris plant. One class is linearly separable from the other two, which are not linearly separable from each other (see Figure~\ref{fig:real_data}). The label set in Iris dataset is $L = \{0, 1, 2\}$ and their corresponding names are ``Setosa", ``Versicolor", and ``Virginica", respectively. Each sample in the Iris dataset is a 5-tuple, defined by (sepal\_length, sepal\_width, petal\_length, petal\_width, label).  
The class of Iris plant will be the predicted attribute. The point set $P$ is built using the first four components of each sample. \par

The \textbf{\textit{Wine dataset}}~\cite{Duadua:2019} is the result of a chemical analysis of wines grown in the same region in Italy by three different growers. There are thirteen different measurements taken for different components found in the three types of wine. The label set $L = \{0, 1, 2\}$ to enumerate the three wine types will be taken from the first component of each sample.~The point set $P \subset \mathbb{R}^{13}$ will be completed using the remaining 13 components of each sample. Figure~\ref{fig:real_data} shows the Wine dataset samples distribution after applying PCA to reduce dimensions to 3, and it was plotted combining two dimensions.\par

The \textbf{\textit{Breast Cancer dataset}}~\cite{Duadua:2019} features were computed from a digitized image of a fine needle aspirate (FNA) of a breast mass. They describe the characteristics of the cell nuclei present in the image. The label set $L=\{0, 1\}$ will denote Malignant (0) tumors and Benignant (1) tumors. The point set will be $P \subset \mathbb{R}^{30}$ where each sample represents the cell nuclei information of one image. Figure~\ref{fig:real_data} shows this dataset after applying a PCA process to visualize it from several 2-dimensional perspectives. \par

\begin{figure}[H]
\centering
\includegraphics[width=0.8\textwidth]{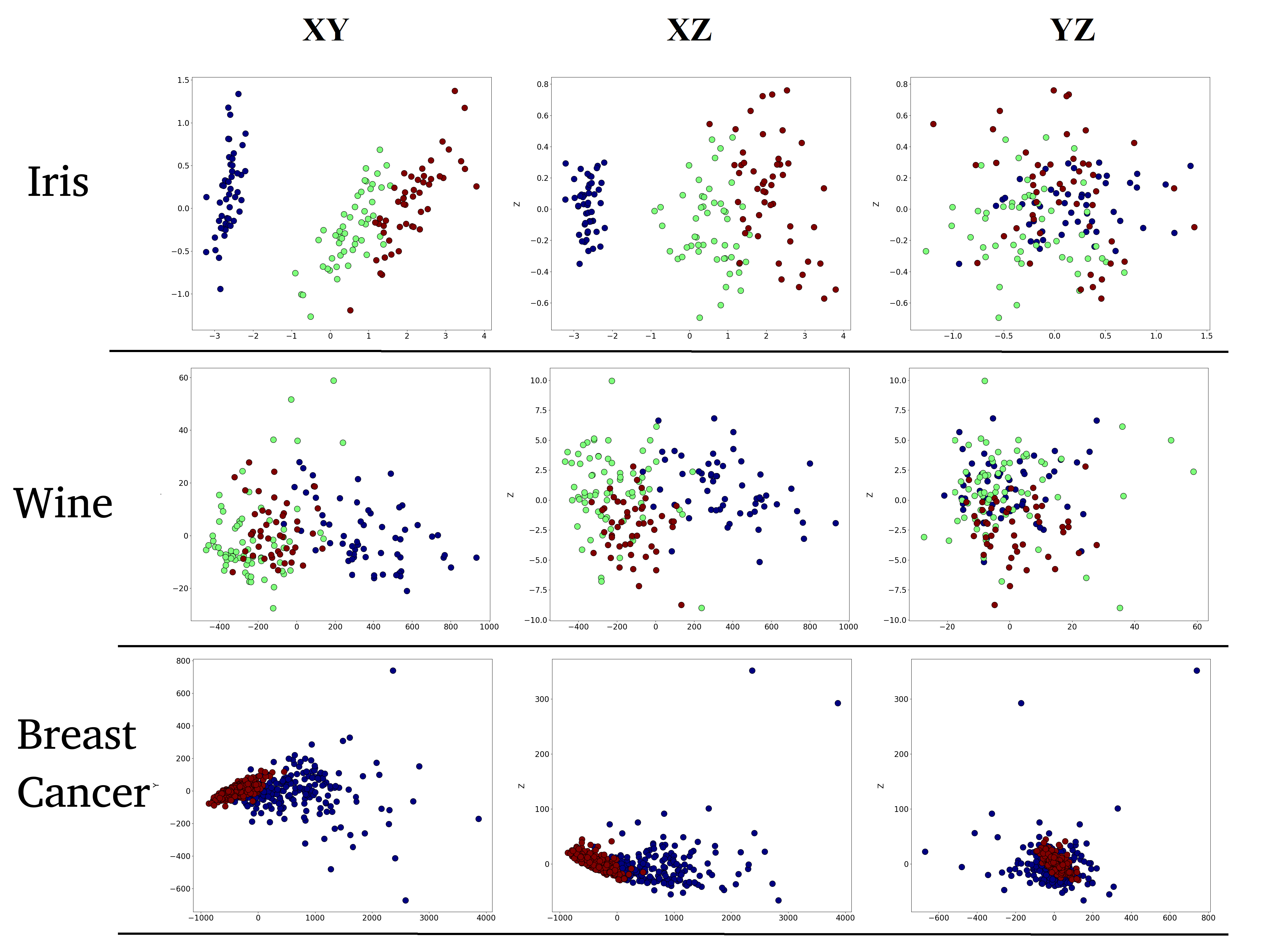}
\caption{The Iris, Wine, and Breast Cancer datasets, chosen as real datasets. }\label{fig:real_data}
\end{figure}

\subsection{Classifier Evaluation}

The classifier evaluation over all datasets were conducted using a Repeated cross-validation process (see~\ref{app:rcv} for details). The aim is to avoid biased results because the training and test sets are each time from the same dataset. 

Let $R$ be the fold size in the Repeated Cross-Validation approach to avoid confusion with the use of k in k-NN. The R-FOLD Cross Validation is then repeated 5 times (N=5), and R will be the 10\% of the selected dataset. For any value of R, $(\left \lceil \frac{P}{R}  \right\rceil -1)$ R-folds will be selected to be the training set $S$, and the remaining R-fold will be the test data $X$ in each iteration. When $2 \cdot R \geq \abs{P}$ the problem is considered to be semi-supervised, where there are more unknown samples to classify than there are labeled samples.  \par
It is common in ML algorithms to use parameters whose values are changed before the learning process begins. Those parameters are called hyper-parameters~\cite{scikit-learn, Japkowicz2011ELA, tom97ML}. For k-NN and wk-NN algorithms, a k=15 was considered a good number of neighbors; we obtained it using the hyper-parameter estimators from scikit-learn~\cite{scikit-learn}. For the three TDABC algorithms, the maximal simplex dimension needs to be fixed to a $q$ value to control the VR-complex construction process. Experiments were conducted with $3 \leq q \leq 10$. 
In~\ref{app:metrics}, a detailed explanation of selected metrics is given. Results presented were obtained on two computers: 8 GB RAM, Intel Core i7-6500U CPU 2.50GHz x 4, and 16 GB RAM, $\mbox{AMD Ryzen}^{TM}$ 7 3700U 2.30 GHz x 4.

\subsection{Comparison}

Each classifier is executed a total number of $W = N \cdot \ceil*{\frac{\abs{P}}{R}}$ times because of the repeated cross validation process. This process results in a total number of $W$ predicted collections $\hat{Y}_1, \hat{Y}_2, \hat{Y}_3, \cdots, \hat{Y}_W $ for each classifier. Similarly, a total number of real label collections $Y_1,Y_2, \cdots, Y_W$ results for each classifier. Those two lists of collections, $\{\hat{Y}_i\}_{1\leq i \leq W}$, and $\{Y_i\}_{1 \leq i \leq W}$ are concatenated by putting the collection $i+1$ at the end of the previous collection, which results in two big collections of predicted, and real labels.

    \begin{IEEEeqnarray}{rCl}
 \hat{Y} & = &(\hat{y}_1, \hat{y}_2, \cdots,\hat{y}_{\abs{P}}, \hat{y}_{\abs{P}+1}, \cdots, \hat{y}_{2\cdot \abs{P}},\cdots, \hat{y}_{N \cdot \abs{P}}), \\       \label{eq:pred}
Y & = & (y_1, y_2, \cdots, y_{\abs{P}}, y_{\abs{P}+1}, \cdots, y_{2 \cdot \abs{P}}, \cdots, y_{N \cdot \abs{P}}),         \label{eq:real}
\end{IEEEeqnarray}

Where $\hat{Y}_i = (\hat{y}_j)_{(i-1) \cdot \abs{P}+1 \leq j \leq i\cdot \abs{P}}$ is the predicted labels list and $Y_i = (y_j)_{(i-1) \cdot \abs{P}+1 \leq j \leq i\cdot \abs{P}}$ the real labels list, both resulting from $i^{th}$ execution of Repeated R-Fold. As the correspondence between each components of both predicted and real labels is maintained, it is easy to generalize all metrics' computation by considering $n = \abs{\hat{Y}} = \abs{Y}$. 

\subsubsection{General metrics computation result}

In Table~\ref{tb:metrics_results_001} and Table~\ref{tb:metrics_results_002}, all metrics results are shown. Each metric  was computed across all datasets. For each metric, columns from 2 to 10 represent the datasets, and rows represent each classifier results. The two last columns show the arithmetic mean and standard deviation of the corresponding metrics across all datasets. More details about the metrics' computations are presented in~\ref{app:metrics}.\par

The experiments were conducted per dataset for each fixed simplicial complex dimension on the interval $q \in [3, 9]$. Nevertheless, in this section, results are shown for one value of $q$ in each dataset. For example, from the Iris, Circles, and Sphere datasets, a fixed simplicial complex dimension $q=8$ was selected. As another example, for the Moon-dataset, results were selected for experiments conducted using a simplicial complex dimension $q=3$. In contrast, metrics results on the Swissroll and Wine datasets were selected for simplicial complex dimension $q=6$. On the other hand, the presented results were obtained for the Breast Cancer dataset using a fixed simplicial complex dimension $q = 4$. Meanwhile, from the Normdist dataset, results were selected using a fixed simplicial complex dimension $q=7$.\par

\begin{table}[]
\centering
\resizebox{\textwidth}{!}{%
\begin{tabular}{|l|l|l|l|l|l|l|l|l||l|l|}
\hline
\multicolumn{11}{|l|}{\textbf{Accuracy (ACC)}} \\ \hline
\textbf{Method} & \textbf{Circles} & \textbf{Moon} & \textbf{Swissroll} & \textbf{\begin{tabular}[c]{@{}l@{}}Norm\\ dist\end{tabular}} & \textbf{Sphere} & \textbf{IRIS} & \textbf{Wine} & \textbf{\begin{tabular}[c]{@{}l@{}}Breast \\ cancer\end{tabular}} & \textbf{Mean} & \textbf{Stdev} \\ \hline
\textbf{TDABC-A} & 0,668 & 0,516 & 0,919 & \textbf{0,993} & 0,918 & 0,961 & \textbf{0,769} & 0,91 & \textbf{0,832} & 0,167 \\ \hline
\textbf{TDABC-M} & \textbf{0,700} & 0,524 & 0,911 & 0,983 & \textbf{0,938} & 0,920 & 0,767 & 0,92 & \textbf{0,832} & 0,157 \\ \hline
\textbf{TDABC-R} & 0,676 & \textbf{0,530} & 0,913 & 0,990 & 0,925 & 0,936 & 0,768 & 0,92 & \textbf{0,832} & 0,159 \\ \hline
\textbf{wk-NN} & 0,468 & 0,456 & \textbf{0,943} & 0,990 & 0,901 & 0,977 & 0,739 & \textbf{0,93} & 0,801 & 0,223 \\ \hline
\textbf{k-NN} & 0,508 & 0,475 & 0,918 & 0,987 & 0,887 & \textbf{0,980} & 0,708 & \textbf{0,93} & 0,799 & 0,209 \\ \hline
\multicolumn{11}{|l|}{\textbf{Precision (PR)}} \\ \hline
\textbf{Method} & \textbf{Circles} & \textbf{Moon} & \textbf{Swissroll} & \textbf{\begin{tabular}[c]{@{}l@{}}Norm\\ dist\end{tabular}} & \textbf{Sphere} & \textbf{IRIS} & \textbf{Wine} & \textbf{\begin{tabular}[c]{@{}l@{}}Breast \\ cancer\end{tabular}} & \textbf{Mean} & \textbf{Stdev} \\ \hline
\textbf{TDABC-A} & 0,668 & 0,516 & 0,920 & \textbf{0,979} & 0,718 & 0,960 & \textbf{0,765} & 0,91 & 0,805 & 0,165 \\ \hline
\textbf{TDABC-M} & \textbf{0,700} & 0,524 & 0,911 & 0,971 & \textbf{0,747} & 0,917 & 0,763 & \textbf{0,92} & 0,806 & 0,151 \\ \hline
\textbf{TDABC-R} & 0,676 & \textbf{0,530} & 0,913 & 0,976 & 0,742 & 0,934 & 0,764 & \textbf{0,92} & \textbf{0,807} & 0,155 \\ \hline
\textbf{wk-NN} & 0,468 & 0,456 & \textbf{0,940} & 0,950 & 0,579 & 0,976 & 0,742 & \textbf{0,92} & 0,754 & 0,224 \\ \hline
\textbf{k-NN} & 0,508 & 0,475 & 0,914 & 0,933 & 0,549 & \textbf{0,979} & 0,701 & \textbf{0,92} & 0,747 & 0,213 \\ \hline
\multicolumn{11}{|l|}{\textbf{Recall (RE)}} \\ \hline
\textbf{Method} & \textbf{Circles} & \textbf{Moon} & \textbf{Swissroll} & \textbf{\begin{tabular}[c]{@{}l@{}}Norm\\ dist\end{tabular}} & \textbf{Sphere} & \textbf{IRIS} & \textbf{Wine} & \textbf{\begin{tabular}[c]{@{}l@{}}Breast \\ cancer\end{tabular}} & \textbf{Mean} & \textbf{Stdev} \\ \hline
\textbf{TDABC-A} & 0,668 & 0,516 & 0,692 & 0,964 & 0,454 & 0,927 & \textbf{0,622} & 0,90 & \textbf{0,718} & 0,193 \\ \hline
\textbf{TDABC-M} & \textbf{0,700} & 0,524 & 0,667 & 0,901 & \textbf{0,496} & 0,859 & 0,619 & 0,91 & 0,709 & 0,164 \\ \hline
\textbf{TDABC-R} & 0,676 & \textbf{0,530} & 0,673 & 0,941 & 0,472 & 0,885 & 0,621 & 0,91 & 0,713 & 0,178 \\ \hline
\textbf{wk-NN} & 0,463 & 0,456 & \textbf{0,764} & \textbf{0,966} & 0,400 & 0,957 & 0,590 & \textbf{0,93} & 0,691 & 0,243 \\ \hline
\textbf{k-NN} & 0,509 & 0,475 & 0,688 & 0,955 & 0,379 & \textbf{0,962} & 0,547 & \textbf{0,93} & 0,681 & 0,239 \\ \hline
\multicolumn{11}{|l|}{\textbf{True Negative Rate (TNR)}} \\ \hline
\textbf{Method} & \textbf{Circles} & \textbf{Moon} & \textbf{Swissroll} & \textbf{\begin{tabular}[c]{@{}l@{}}Norm\\ dist\end{tabular}} & \textbf{Sphere} & \textbf{IRIS} & \textbf{Wine} & \textbf{\begin{tabular}[c]{@{}l@{}}Breast \\ cancer\end{tabular}} & \textbf{Mean} & \textbf{Stdev} \\ \hline
\textbf{TDABC-A} & 0,668 & 0,516 & 0,983 & \textbf{0,999} & 0,981 & 0,981 & \textbf{0,873} & 0,90 & 0,863 & 0,178 \\ \hline
\textbf{TDABC-M} & \textbf{0,700} & 0,524 & 0,981 & 0,996 & \textbf{0,985} & 0,961 & 0,871 & 0,91 & \textbf{0,866} & 0,169 \\ \hline
\textbf{TDABC-R} & 0,676 & \textbf{0,530} & 0,981 & 0,998 & 0,982 & 0,969 & 0,872 & 0,91 & 0,864 & 0,171 \\ \hline
\textbf{wk-NN} & 0,463 & 0,456 & \textbf{0,988} & 0,998 & 0,977 & 0,989 & 0,859 & \textbf{0,93} & 0,833 & 0,235 \\ \hline
\textbf{k-NN} & 0,509 & 0,475 & 0,983 & 0,997 & 0,974 & \textbf{0,990} & 0,841 & \textbf{0,93} & 0,838 & 0,220 \\ \hline
\multicolumn{11}{|l|}{\textbf{False Positive Rate (FPR)}} \\ \hline
\textbf{Method} & \textbf{Circles} & \textbf{Moon} & \textbf{Swissroll} & \textbf{\begin{tabular}[c]{@{}l@{}}Norm\\ dist\end{tabular}} & \textbf{Sphere} & \textbf{IRIS} & \textbf{Wine} & \textbf{\begin{tabular}[c]{@{}l@{}}Breast \\ cancer\end{tabular}} & \textbf{Mean} & \textbf{Stdev} \\ \hline
\textbf{TDABC-A} & 0,332 & 0,484 & 0,080 & \textbf{0,021} & 0,282 & 0,040 & \textbf{0,235} & 0,09 & 0,195 & 0,165 \\ \hline
\textbf{TDABC-M} & \textbf{0,300} & 0,476 & 0,089 & 0,029 & \textbf{0,253} & 0,083 & 0,237 & \textbf{0,08} & 0,194 & 0,151 \\ \hline
\textbf{TDABC-R} & 0,324 & \textbf{0,470} & 0,087 & 0,024 & 0,258 & 0,066 & 0,236 & \textbf{0,08} & \textbf{0,193} & 0,155 \\ \hline
\textbf{wk-NN} & 0,532 & 0,544 & \textbf{0,060} & 0,050 & 0,421 & 0,024 & 0,258 & \textbf{0,08} & 0,246 & 0,224 \\ \hline
\textbf{k-NN} & 0,492 & 0,525 & 0,086 & 0,067 & 0,452 & \textbf{0,021} & 0,295 & \textbf{0,08} & 0,252 & 0,213 \\ \hline
\end{tabular}%
}
\caption{The Acc, Pr, Re, TNR, and FPR metric results per classifier across all datasets.}
\label{tb:metrics_results_001}
\end{table}

\begin{table}[]
\centering
\resizebox{\textwidth}{!}{%
\begin{tabular}{|l|l|l|l|l|l|l|l|l||l|l|}
\hline
\multicolumn{11}{|l|}{\textbf{F1-Measure (F1)}} \\ \hline
\textbf{Method} & \textbf{Circles} & \textbf{Moon} & \textbf{Swissroll} & \textbf{\begin{tabular}[c]{@{}l@{}}Norm\\ dist\end{tabular}} & \textbf{Sphere} & \textbf{IRIS} & \textbf{Wine} & \textbf{\begin{tabular}[c]{@{}l@{}}Breast \\ cancer\end{tabular}} & \textbf{Mean} & \textbf{Stdev} \\ \hline
\textbf{TDABC-A} & 0,668 & 0,516 & 0,789 & \textbf{0,971} & 0,504 & 0,943 & \textbf{0,683} & 0,91 & \textbf{0,748} & 0,185 \\ \hline
\textbf{TDABC-M} & \textbf{0,700} & 0,524 & 0,770 & 0,934 & \textbf{0,560} & 0,883 & 0,680 & 0,91 & 0,745 & 0,157 \\ \hline
\textbf{TDABC-R} & 0,676 & \textbf{0,530} & 0,774 & 0,958 & 0,527 & 0,906 & 0,682 & 0,91 & 0,745 & 0,170 \\ \hline
\textbf{wk-NN} & 0,450 & 0,455 & \textbf{0,842} & 0,954 & 0,433 & 0,966 & 0,648 & \textbf{0,93} & 0,709 & 0,240 \\ \hline
\textbf{k-NN} & 0,494 & 0,475 & 0,784 & 0,937 & 0,524 & \textbf{0,970} & 0,608 & \textbf{0,93} & 0,715 & 0,213 \\ \hline
\multicolumn{11}{|l|}{\textbf{Matthews Correlation Coefficient (MCC)}} \\ \hline
\textbf{Method} & \textbf{Circles} & \textbf{Moon} & \textbf{Swissroll} & \textbf{\begin{tabular}[c]{@{}l@{}}Norm\\ dist\end{tabular}} & \textbf{Sphere} & \textbf{IRIS} & \textbf{Wine} & \textbf{\begin{tabular}[c]{@{}l@{}}Breast \\ cancer\end{tabular}} & \textbf{Mean} & \textbf{Stdev} \\ \hline
\textbf{TDABC-A} & 0,336 & 0,032 & 0,752 & \textbf{0,967} & 0,478 & 0,915 & \textbf{0,514} & 0,82 & \textbf{0,601} & 0,321 \\ \hline
\textbf{TDABC-M} & \textbf{0,400} & 0,048 & 0,729 & 0,925 & \textbf{0,541} & 0,828 & 0,510 & 0,82 & 0,600 & 0,288 \\ \hline
\textbf{TDABC-R} & 0,352 & \textbf{0,060} & 0,735 & 0,952 & 0,506 & 0,861 & 0,512 & 0,82 & 0,600 & 0,300 \\ \hline
\textbf{wk-NN} & -0,069 & -0,088 & \textbf{0,815} & 0,950 & 0,384 & 0,950 & 0,465 & \textbf{0,86} & 0,533 & 0,432 \\ \hline
\textbf{k-NN} & 0,017 & -0,050 & 0,747 & 0,932 & 0,346 & \textbf{0,955} & 0,400 & 0,85 & 0,525 & 0,404 \\ \hline
\multicolumn{11}{|l|}{\textbf{Geometric Mean (GMean)}} \\ \hline
\textbf{Method} & \textbf{Circles} & \textbf{Moon} & \textbf{Swissroll} & \textbf{\begin{tabular}[c]{@{}l@{}}Norm\\ dist\end{tabular}} & \textbf{Sphere} & \textbf{IRIS} & \textbf{Wine} & \textbf{\begin{tabular}[c]{@{}l@{}}Breast \\ cancer\end{tabular}} & \textbf{Mean} & \textbf{Stdev} \\ \hline
\textbf{TDABC-A} & 0,668 & 0,516 & 0,824 & 0,981 & 0,614 & 0,954 & \textbf{0,735} & 0,90 & \textbf{0,774} & 0,169 \\ \hline
\textbf{TDABC-M} & \textbf{0,700} & 0,524 & 0,809 & 0,946 & \textbf{0,656} & 0,908 & 0,733 & 0,90 & 0,773 & 0,146 \\ \hline
\textbf{TDABC-R} & 0,676 & \textbf{0,530} & 0,812 & 0,969 & 0,632 & 0,925 & 0,734 & 0,90 & 0,773 & 0,156 \\ \hline
\textbf{wk-NN} & 0,463 & 0,456 & \textbf{0,869} & \textbf{0,982} & 0,522 & 0,973 & 0,709 & \textbf{0,93} & 0,738 & 0,231 \\ \hline
\textbf{k-NN} & 0,509 & 0,475 & 0,822 & 0,976 & 0,500 & \textbf{0,976} & 0,676 & \textbf{0,93} & 0,733 & 0,221 \\ \hline
\multicolumn{11}{|l|}{\textbf{Classification Error (CError)}} \\ \hline
\textbf{Method} & \textbf{Circles} & \textbf{Moon} & \textbf{Swissroll} & \textbf{\begin{tabular}[c]{@{}l@{}}Norm\\ dist\end{tabular}} & \textbf{Sphere} & \textbf{IRIS} & \textbf{Wine} & \textbf{\begin{tabular}[c]{@{}l@{}}Breast \\ cancer\end{tabular}} & \textbf{Mean} & \textbf{Stdev} \\ \hline
\textbf{TDABC-A} & 0,332 & 0,484 & 0,081 & \textbf{0,007} & 0,083 & 0,039 & \textbf{0,231} & 0,09 & \textbf{0,168} & 0,167 \\ \hline
\textbf{TDABC-M} & \textbf{0,300} & 0,476 & 0,089 & 0,017 & \textbf{0,062} & 0,080 & 0,233 & 0,08 & \textbf{0,168} & 0,157 \\ \hline
\textbf{TDABC-R} & 0,324 & \textbf{0,470} & 0,081 & 0,010 & 0,075 & 0,064 & 0,232 & 0,08 & \textbf{0,168} & 0,160 \\ \hline
\textbf{wk-NN} & 0,532 & 0,544 & \textbf{0,057} & 0,010 & 0,099 & 0,023 & 0,261 & \textbf{0,07} & 0,199 & 0,223 \\ \hline
\textbf{k-NN} & 0,492 & 0,525 & 0,082 & 0,013 & 0,113 & \textbf{0,020} & 0,292 & \textbf{0,07} & 0,201 & 0,209 \\ \hline
\end{tabular}%
}
\caption{The F1, MCC, GMean, and CErr metric results per classifier across all datasets.}
\label{tb:metrics_results_002}
\end{table}

Table~\ref{tb:summary_table} summarizes the classifiers' average performance. It was built using the two last columns of Tables~\ref{tb:metrics_results_001} and Table~\ref{tb:metrics_results_002}, which represent the mean and the standard deviation of each metric across all datasets. 

\begin{table}[]
\centering
\resizebox{\textwidth}{!}{%
\begin{tabular}{|l|l|l|l|l|l|}
\hline
\textbf{Name} & \textbf{TDABC-A} & \textbf{TDABC-M} & \textbf{TDABC-R} & \textbf{wk-NN} & \textbf{k-NN} \\ \hline
\textbf{Acc} & \boldmath $0,832\pm 0,17$ & \boldmath $0,832\pm 0,16$ & \boldmath $0,832\pm 0,16$ & $0,801\pm 0,22$ & $0,799\pm 0,21$ \\ \hline
\textbf{Pr} & $0,805\pm 0,17$ & $0,806 \pm 0,15$ & \boldmath {$0,807\pm 0,15$} & $0,754\pm 0,22$ & $0,747\pm 0,21$ \\ \hline
\textbf{Re} & \boldmath {$0,718\pm 0,19$} & $0,709\pm 0,16$ & $0,713\pm 0,18$ & $0,691\pm 0,24$ & $0,681\pm 0,24$ \\ \hline
\textbf{TNR} & $0,863\pm 0,18$ & \boldmath{$0,866\pm 0,17$} & $0,864 \pm 0,17$ & $0,833 \pm 0,23$ & $0,837\pm 0,22$ \\ \hline
\textbf{FPR} & $0,195\pm 0,17$ & $0,194\pm 0,15$ & \boldmath{$0,193\pm 0,15$} & $0,246 \pm 0,22$ & $0,252 \pm 0,21$ \\ \hline
\textbf{F1} & \boldmath{$0,75 \pm 0,18$} & \boldmath{$0,75\pm 0,16$} & \boldmath{$0,75\pm 0,17$} & $0,71\pm 0,24$ & $0,71 \pm 0,21$ \\ \hline
\textbf{MCC} & \boldmath{$0,601\pm 0,32$} & $0,600\pm 0,29$ & $0,600 \pm 0,30$ & $0,533\pm 0,43$ & $0,525\pm 0,40$ \\ \hline
\textbf{GMEAN} & \boldmath{$0,774\pm 0,17$} & $0,773\pm 0,15$ & $0,773 \pm 0,16$ & $0,738\pm 0,23$ & $0,733\pm 0, 22$ \\ \hline
\textbf{CErr} & \boldmath{$0,168\pm 0,17$} & \boldmath{$0,168 \pm 0,16$} & \boldmath{$0,168 \pm 0,16$} & $0,199\pm 0,22$ & $0,201\pm 0,21$ \\ \hline
\end{tabular}%
}
\caption{Summary table with the arithmetic mean of the classifiers across all analyzed data sets. Each mean result and standard deviation is shown for the performance in this metric across all datasets.}
\label{tb:summary_table}
\end{table}

\subsubsection{Selected confusion matrices}

For a graphical visualization of the evaluated classifiers' performance, 40 confusion matrices were created, each one corresponding to a classifier in a dataset (5 classifiers, 8 datasets). Nonetheless, only matrices for Iris, Circles, Moon and Sphere datasets are shown in this section. All confusion matrices can be seen in~\ref{app:matrices}. 
\begin{figure}[H]
\begin{center}
\includegraphics[width=0.80\textwidth]{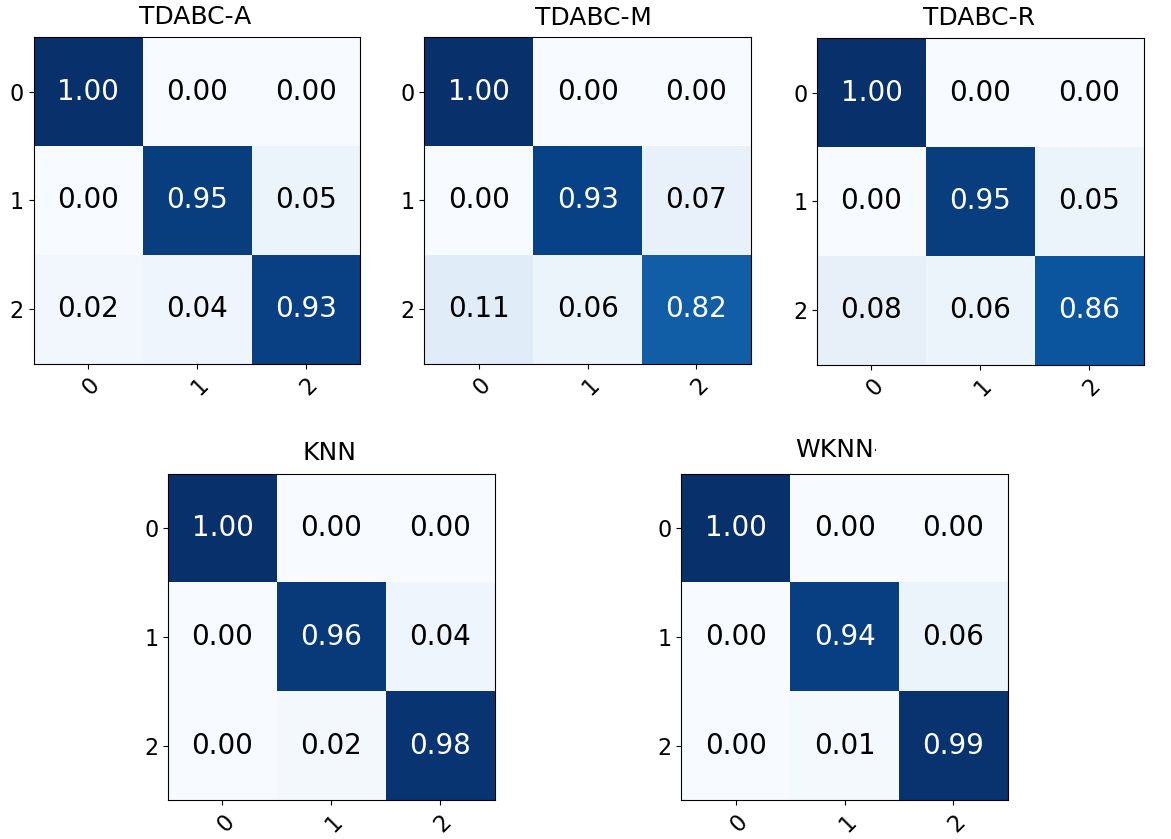}

\caption{Confusion matrices from classifiers' results on the Iris dataset. The filtered simplicial complexes used by TDABC classifiers were built up to a maximal dimension $q=8$.}\label{fig:iris_CM}    
\end{center}
\end{figure}

\begin{figure}[H]
\begin{center}
\includegraphics[width=0.80\textwidth]{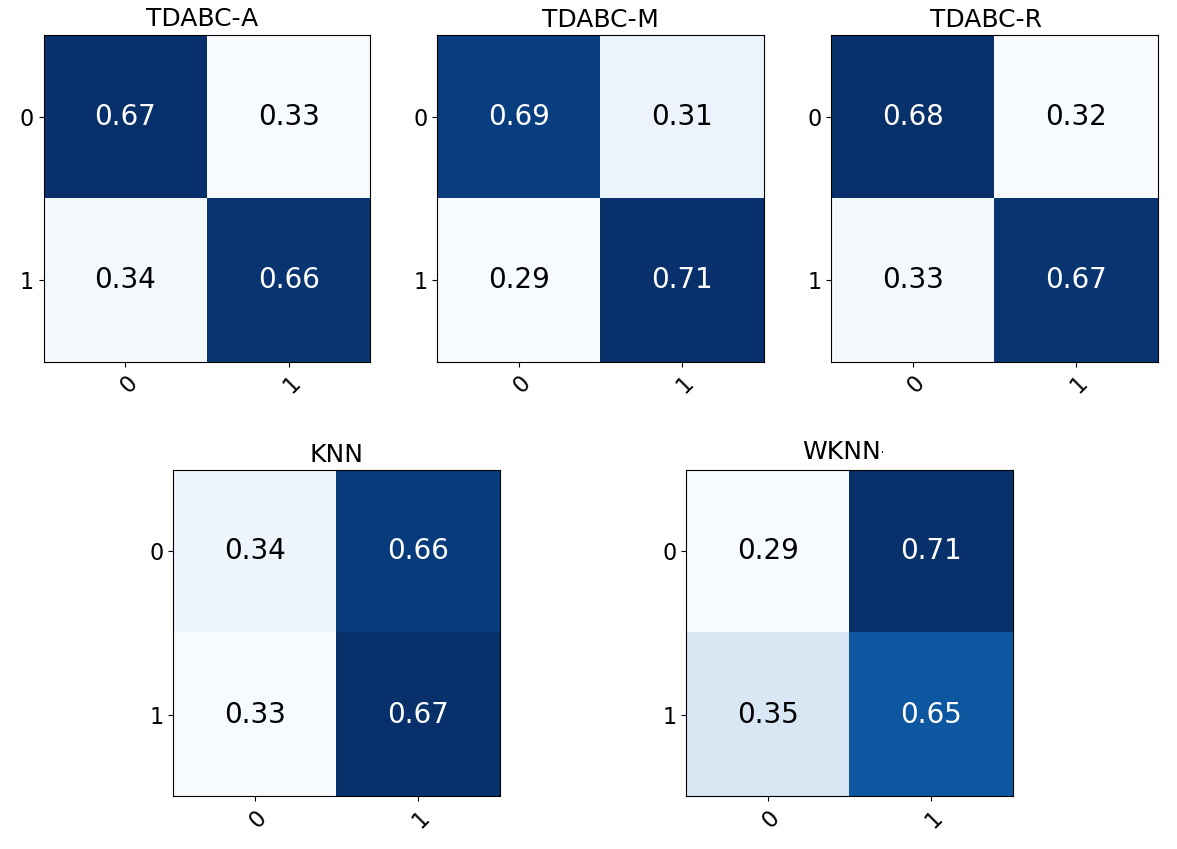}
\caption{Confusion matrices from classifiers' results on the Circles dataset. The filtered simplicial complexes used by TDABC classifiers were built up to a maximal dimension $q=8$.}\label{fig:ALL_circle_CM}    
\end{center}
\end{figure}

\begin{figure}[H]
\begin{center}
\includegraphics[width=0.80\textwidth]{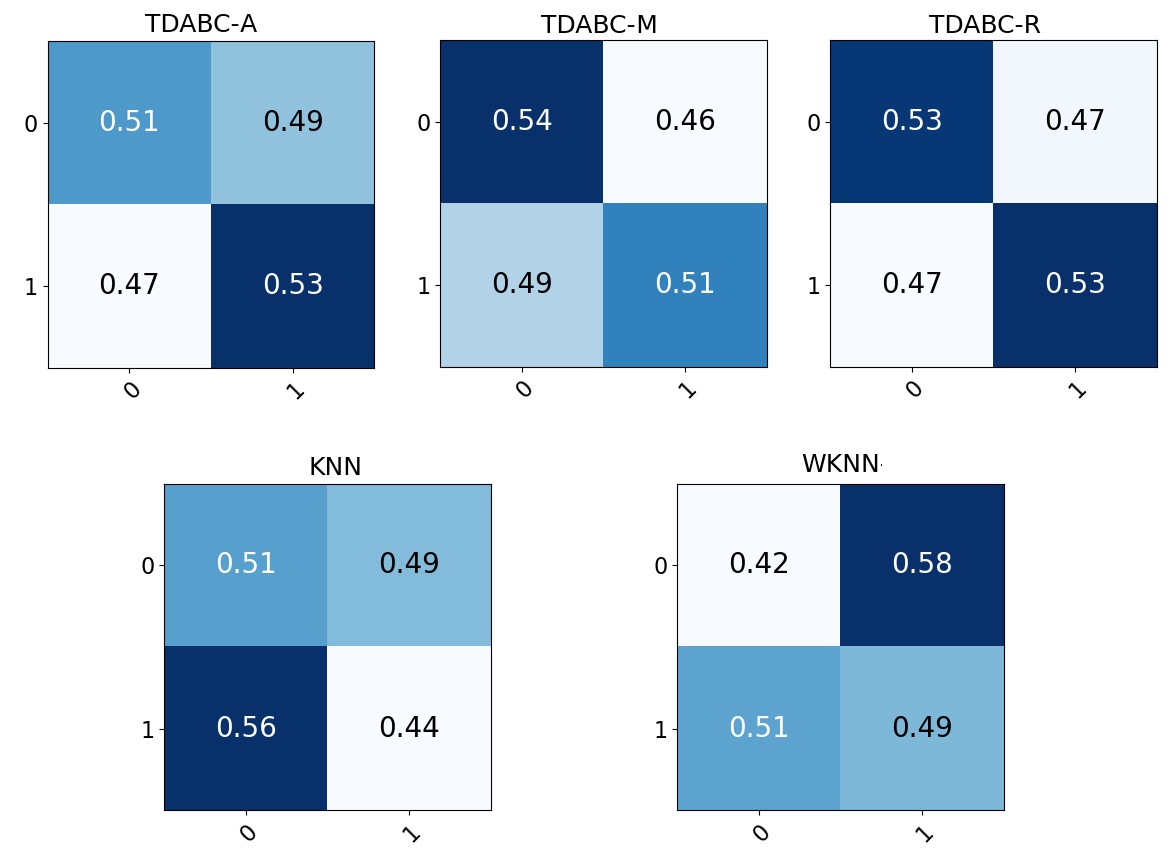}
\caption{Confusion matrices from classifiers' results on the Moon dataset. The filtered simplicial complexes used by TDABC classifiers were built up to a maximal dimension $q=3$.}\label{fig:ALL_moon_CM}    
\end{center}
\end{figure}

\begin{figure}[H]
\begin{center}
\includegraphics[width=0.80\textwidth]{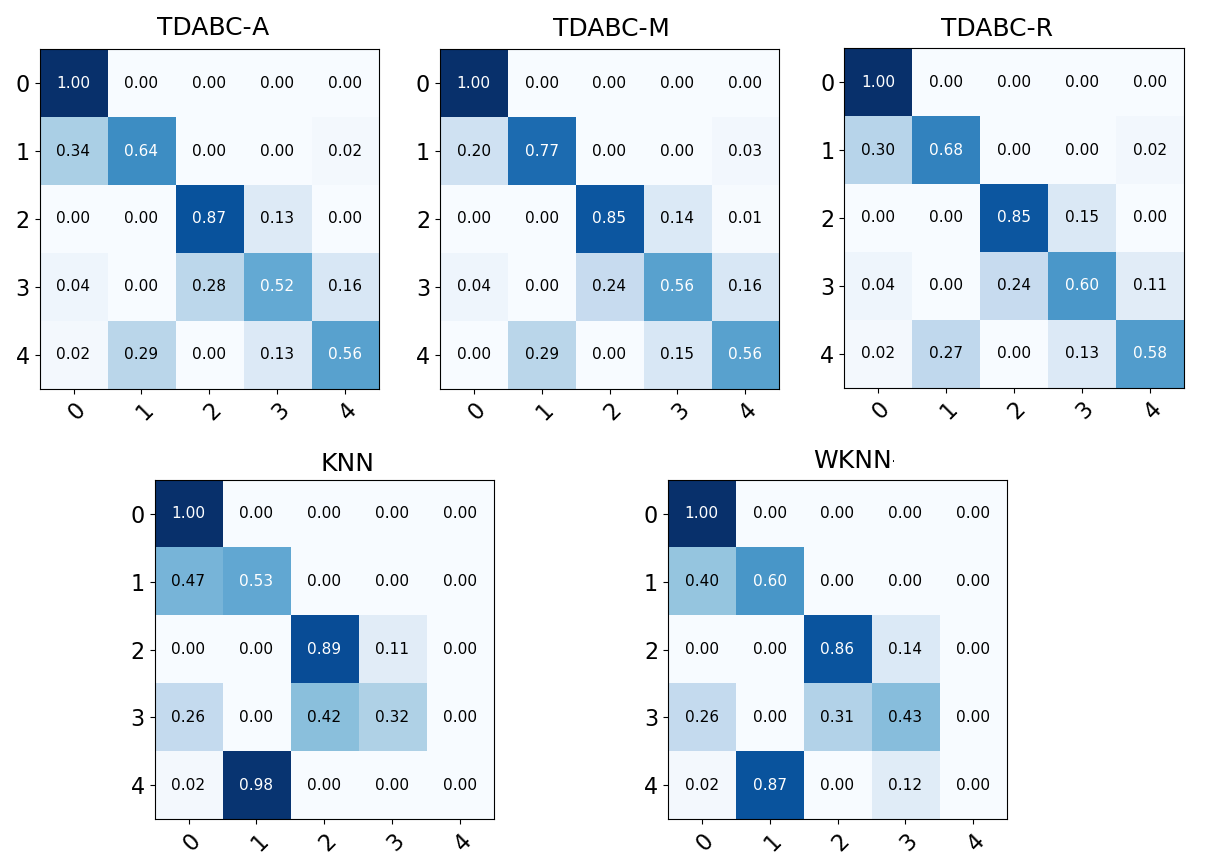}
\caption{Confusion matrices from classifiers' results on the Sphere dataset. The filtered simplicial complexes used by TDABC classifiers were built up to a maximal dimension $q=8$.}\label{fig:ALL_CM1}    
\end{center}
\end{figure}
  
\section{Discussion}\label{scc:discussion}

The discussion section is organized into three subsections. First, the analysis of results, followed by a highlight of the proposed method's most relevant characteristics and a discussion of related works.

\subsection{Results}

Average performance was computed with the arithmetic mean, geometric mean, and harmonic mean in all datasets. Algorithm ranking remained across means' computations; thus, only arithmetic mean is shown in Table~\ref{tb:summary_table}. \par
By analyzing results per dataset independently, it could be noted that the TDABC approaches were superior to wk-NN and k-NN in 5 of the 8 evaluated datasets, specifically on the Circles, Moon, Normdist, Sphere, and Wine datasets (see Table~\ref{tb:metrics_results_001} and Table~\ref{tb:metrics_results_002}). On the other hand, baseline methods were slightly better on the three remaining datasets.\par
The Circles, Moon, Normdist, Sphere, and Wine datasets have different challenging features such as high dimensionality, the imbalanced distribution of labels, and highly entanglement classes. Despite these challenges, TDABC approaches overcome baseline methods in every computed metric.\par

The Circles and Moon datasets are balanced and have very entangled classes due to the noise factor, making the classification a challenge. In these datasets, wk-NN and k-NN behave poorly, as observed through the negative values obtained after applying the MCC measure (see Table~\ref{tb:metrics}). This behavior is related to the fixed value of k and to the assumption that each data point is equally relevant. Even though wk-NN imposes a local data point weighting based on distances, it is not enough with highly entangled classes, as our results show. The TDABC methods are capable of dealing with the entanglement challenge through a disambiguation factor based on filtration values ($\xi_{\mathcal{K}}$). \par
In the case of the Normdist and Sphere datasets, there is a high imbalanced ratio of the classes, with $[60, 10, 50, 100, 80]$ and $[500, 100, 25, 16, 12]$ samples per class, respectively. In this situation, it is important to have dynamic neighborhoods. The proposed method generates dynamic-sized ``neighborhoods'' for each point, in contrast to~k-NN and~wk-NN classifiers. In the high imbalance case, the disambiguation factor  ($\xi_{\mathcal{K}}$) also provides a multi-scale local weighting to TDABC methods.\par
The Wine and Normdist datasets are highly dimensional (13 and 350 dimensions, respectively). TDABC methods behave better than baseline approaches for those datasets in all metrics. k-NN and wk-NN use Euclidean distance directly to detect their k neighbors. Even though TDABC methods also use Euclidean distance, they are still able to unravel multi-scale, multi-dimensional relationships among data, which better handle high dimensionality. \par

Swissroll and Iris are balanced datasets, with $50$ samples per class. In contrast, the Breast Cancer dataset has $[212, 357]$ samples per class. In the Swissroll and Iris datasets, weighted k-NN and k-NN were better, respectively, than proposed classifiers in all metrics. Interestingly, for the Breast Cancer dataset, which is slightly imbalanced, TDABC was equally performant in 2 out of 9 metrics.\par

\subsection{Key aspects of the proposed method}

Regarding the proposed TDA-based classification methodology, two key aspects are discussed: persistent homology and voting system. 


The first aspect is related to persistent homology's key role in selecting the desired sub-complex from a filtered simplicial complex built on the dataset. Algorithm~\ref{alg:piD} reduces the search space in the filtered simplicial complex by taking advantage of topological features encoded inside selected persistence intervals.~Although, selecting the right sub-complex is a very challenging problem~\cite{Caillerie2011}, the simple criterion we propose (death time of persistence intervals resulting from MaxInt, RandInt, and AvgInt) is sufficient to achieve good classification results. \par
Despite the birth time's theoretical guarantees to selecting the sub-complex, the middle and death times might be useful depending on the dataset structure and complexity. Experimentally, promissing results were obtained using both middle and death time. However, death time was better experimentally because it can reach more stable topological features and minimize the presence of isolated points. This process is summarized in Figure~\ref{fig:circlepinfo}. \par


The second aspect is the proposed voting system (see Definition~\ref{def:assoc_func_psi}), which gives richer information than the one used in the classification. During the voting system execution, a fundamental stage is the label propagation performed by the labeling function from Definition~\ref{def:lab_func}. The result of the labeling function could also be represented by a contribution vector $\Upsilon(\sigma) \equiv v \in \mathbb{R}^{\abs{L}}$, with each component $i$ being the contribution of each label $l_i \in L$. By normalizing $v$, the probabilities of $\sigma$ to belong to each component's class is obtained. Thus, the voting system provides the probability of each class, allowing, for instance, the use of ensemble techniques.\par  

\subsection{Related methods}

In other related approaches, the authors of~\cite{Zhang2017} proposes the Rare-class Nearest Neighbour (KRNN), a k-NN variant to deal with the sparsity of the positive samples on an imbalanced dataset. The KRNN uses dynamic local query neighborhoods that contain at least k positive nearest neighbors (a member of minority classes). In~\cite{Vuttipittayamongkol2020}, a different approach is proposed to deal with imbalanced datasets, focusing on negative samples (from majority class) in contrast to~\cite{Zhang2017}. They experimentally prove that negative samples on the overlapping region cause the most inaccuracies on classification. Thus, a neighbor-based algorithm is proposed in~\cite{Vuttipittayamongkol2020} that removes negative samples from the overlapped area. \par
Both~\cite{Zhang2017} and~\cite{Vuttipittayamongkol2020} were built to deal with two-classes imbalanced classification problems successfully. However, when applied to multi-class imbalanced problems, several issues arose in both methods, mostly related to the ambiguity of determining if an instance is a positive or a negative one. In multi-classes imbalanced problems, the same class $l_i$ could play both roles simultaneously because it could be a minority class concerning a class $l_j$, but it could be majority respect to a class $l_k$. 
Closely related testing scenarios were the Normdist and Sphere datasets, where the proposed TDABC method was experimentally evaluated. The proposed method obtains good classification rates on minority classes, and it was also able to deal with the overlapping area because of its disentanglement properties.\par
Recent TDA works still consider TDA as a complement of ML tasks. Works as~\cite{ATIENZA2020107509, Riihimaki2020} focus on discovering better ways to transform persistent homology representations into topological features for deep learning pipelines or sophisticated ML methods. In~\cite{ATIENZA2020107509}, the stability of persistent entropy is provided, justifying its application as a useful statistic in topological data analysis. In~\cite{Riihimaki2020}, TDA is applied to bioinformatics by proposing a novel algorithm based on another major TDA tool called the Mapper algorithm, used to visualize and interpret low and high volume of data (see~\cite{Carlsson:Bulletin}), and built a ML classifier on top of the mapper generated graphs. In~\cite{MAJUMDAR2020113868}, Self-Organized Maps were combined with TDA tools to cluster and classify time series in the financial domain with competitive results. In this context, our work is an example of a fully TDA-based approach applied to supervised learning, with a preliminary version shown as two technical reports in~\cite{rolandokindelanmauriciocerdanancyhitschfeld2020, rknmcnh2020}.\par

\section{Conclusions}\label{scc:conclusions}

In this work, TDA was applied directly in a classification problem and evaluated in 8 datasets, including imbalanced and high dimensionality ones, with good results compared to baseline methods. Overall, we show that Topological Data Analysis alone can classify without any ML method. To our knowledge, this is the first study that proposes this approach for classification. \par 
The proposed TDA-based classification method propagates labels from labeled points to unlabeled ones over the built filtered simplicial complex. The filtration values were interpreted as indirect distance indicators to provide a natural disambiguation method to label-contributions. \par 
The use of persistent homology was key to reduce the search space's complexity by providing the topological features needed to select a sub-complex close enough to the data topology and use it for classification. 
 
\section*{Declaration of Competing Interest}

The authors declare that they have no known competing financial interests or personal relationships that could have appeared to influence the work reported in this paper.

\section*{Acknowledgments}

This research work was supported by the National Agency for Research and Development of Chile (ANID), with grants ANID 2018/BECA DOCTORADO NACIONAL-21181978, FONDECYT 1181506, ICN09\_015, and PIA ACT192015. Beca postdoctoral CONACYT (Mexico) also supports this work. The first author would like to thank professor José Carlos Gómez-Larrañaga from CIMAT, Mexico, due to his support and collaboration. 

\appendix

\section{Repeated cross-validation process}~\label{app:rcv}

The performance evaluation of any classifier in a multi-class classification problem is a difficult task. 
One of the most significant issues is ensuring that the assessment does not make any assumption about data distributions or the classifier. Another problem to address is guaranteeing the testing robustness against bias, overfitting, and underfitting. A well-known approach is to use a cross-validation method~\cite{BROWNE2000108,scikit-learn, Japkowicz2011ELA, tom97ML}. Cross-validation aims to divide the data set $P$ into equal pieces or folds of size R, one of those pieces is selected to be the test set $X$, and the ($\frac{\abs{P}}{R}-1$) remaining folds are considered the training set $S$. This process continues until the last fold is selected to be $X$. However, since all folds are taken from the same dataset, sometimes a fold is a test set and, at other times, it is part of the training set. This process makes Cross Validation biased. One way to overcome this issue is by making a repeated cross validation process. This means repeating the R-Fold cross validation process N times. This method is called Repeated Cross-Validation (see  Figure~\ref{fig:CV}).

\begin{figure}[H]
\begin{center}
\includegraphics[width=0.60\textwidth]{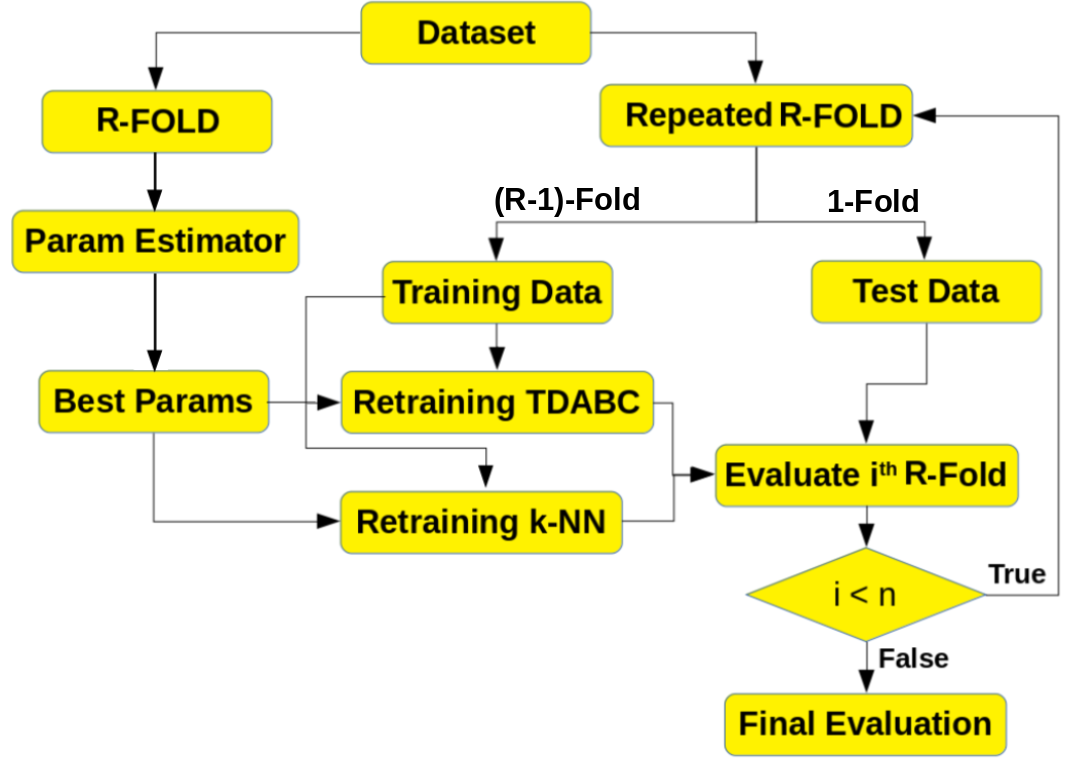}
\caption{Repeated Cross Validation overall process to compare TDABC proposed variants and the baseline classifiers.}\label{fig:CV}    
\end{center}
\end{figure}

\section{Metrics for Classifiers Evaluation}\label{app:metrics}

Several metrics need to be considered to evaluate the proposed and baseline classifiers' performance. The classification metrics are computed as functions of the True Positives ($TP$), True Negatives ($TN$), False Positives ($FP$), and False Negatives ($FN$)'s values. Once those primitive values are computed, it is possible to compute several classification metrics such as Accuracy ($Acc$), Precision ($Pr$), Recall ($Re$), False Positive Rate ($FPR$), False Negative Rates ($FNR$), F1-measure ($F1$ or Harmonic Measure of $Pr$ and $Re$), Matthews Correlation Coefficient ($MCC$), Geometric Mean (GMEAN), and Classification Error ($MSE$). \par
On the other hand, real and predicted label collections definitions are needed to compute the metrics. Let $Y$ be the real label list $\forall x \in X$. Let $\hat{Y}$ be the predicted label list $\forall x \in X$ computed by Algorithm~\ref{alg:TDABC}, where $n = \abs{Y} = \abs{\hat{Y}} = \abs{X}$. The following sections explain the metrics' computation. \par

\subsection{True positives, True negatives, False Positives, and False Negatives}\label{scc:fp}

A True positive sample is a sample successfully classified as belonging to the True label (the critical or most important one). A True Negative value is a sample successfully classified as to be labeled with a negative label. A False positive value is a mislabeled sample with a true label, and a False negative is a mislabeled sample with a negative label. \par  
In a multi-class classification problem (more than two classes), it is more difficult to determine true and negative classes. In this paper, each class is considered a true class, the remaining classes will be the negative ones, and this process is repeated to cover all classes as true classes. In this process, $TP_l, TN_l, FP_l, \mbox{ and } FN_l$ are computed for each label $l \in L$ with L the label set, from Equation~\ref{eq:fp-ilist} to Equation~\ref{eq:fp-elist}

\begin{IEEEeqnarray}{rCl}\label{eq:fp-ilist}
    TP_l & = & \sum_{i = 1}^{n} \mathcal{I}(l = \hat{y}_i) \cdot \mathcal{I}(\hat{y}_i = y_i), \\
    FP_l & = & \sum_{i = 1}^{n} \mathcal{I}(l = \hat{y}_i) \cdot \mathcal{I}(\hat{y}_i \neq y_i), \\
    TN_l & = & \sum_{i = 1}^{n} \mathcal{I}(l \neq \hat{y}_i) \cdot \mathcal{I}(\hat{y}_i = y_i), \\
    FN_l & = & \sum_{i = 1}^{n} \mathcal{I}(l \neq \hat{y}_i) \cdot \mathcal{I}(\hat{y}_i \neq y_i).\label{eq:fp-elist}
\end{IEEEeqnarray}

\subsection{Metrics computation for binary and multi-class classification}\label{scc:acc}

Binary classification is the setting where only two classes are taken into consideration. In this scenario, popular metrics are:

\begin{itemize}
    \item \textbf{Accuracy (Acc):} Percentage of correct predictions over the total samples.
    \item \textbf{Precision:} Number of items correctly identified as positive over the total of positive items.
    \item \textbf{Recall, Sensitivity or True Positive Rate (Re):} Number of items correctly identified as positive out of total true positives.
    \item \textbf{True Negative Rate or Specificity (TNR):} Number of items correctly identified as negative out of total negatives.
    \item \textbf{False Positive Rate or Type I Error (FPR):} Number of items wrongly identified as positive out of total true negatives.
    \item \textbf{False Negative Rate or Type II Error (FNR):} Number of items wrongly identified as negative out of total true positives.
    \item \textbf{F1-Measure:} This measure summarizes Pr and Re in a single metric. It is known to be the harmonic mean from both. It mitigates the impact of the high rate but also accentuates the lower rates' impact.
    \item \textbf{Matthews Correlation Coeficcient (MCC):} A measure unaffected by the imbalanced datasets issue. MCC is a contingency matrix method obtained from calculating the Pearson correlation coefficient between real and predicted values.
    \item \textbf{Geometric Mean (GMean):} The geometric mean corresponds to the square root of the product of the Recall and True Negative Rate. It is commonly used to understand the classifier behavior with imbalanced datasets.     
    \item \textbf{Classification Error (CErr):} Percentage of misclassification over the total samples.
\end{itemize}

\begin{table}[]
\resizebox{\textwidth}{!}{%
\begin{tabular}{|l|c|l|l|l|}
\hline
\textbf{\begin{tabular}[c]{@{}l@{}}Metric\\ Name\end{tabular}} & \multicolumn{1}{l|}{\textbf{Equation}}                                                                                      & \textbf{\begin{tabular}[c]{@{}l@{}}Defined \\ Interval\end{tabular}} & \textbf{Worst} & \textbf{Better} \\ \hline
\textbf{Acc}                                                   & $\frac{1}{\abs{L}}  \cdot  \sum_{l \in L} \frac{TP_l+TN_l}{TP_l+TN_l+FP_l+FN_l}$                                                                                           & {[}0,1{]}                                                            & 0              & 1               \\ \hline
\textbf{Pr}                                                    & $\frac{1}{\abs{L}}  \cdot  \sum_{l \in L} \frac{TP_l}{TP_l+FP_l}$                                            & {[}0,1{]}                                                            & 0              & 1               \\ \hline
\textbf{Re}                                                    & $\frac{1}{\abs{L}}  \cdot  \sum_{l \in L} \frac{TP_l}{TP_l+FN_l}$                                                           & {[}0,1{]}                                                            & 0              & 1               \\ \hline
\textbf{TNR}                                                   & $\frac{1}{\abs{L}}  \cdot  \sum_{l \in L} \frac{TN_l}{TN_l+FP_l}$                                                           & {[}0,1{]}                                                            & 0              & 1               \\ \hline
\textbf{FPR}                                                   & $\frac{1}{\abs{L}}  \cdot  \sum_{l \in L} \frac{FP_l}{TN_l+FP_l}$                                                           & {[}0,1{]}                                                            & 1              & 0               \\ \hline
\textbf{F1}                                                    & $\frac{1}{\abs{L}}  \cdot  \sum_{l \in L} \frac{2 \cdot TP_l}{2 \cdot TP_l + FP_l+FN_l}$ & {[}0,1{]}                                                            & 0              & 1               \\ \hline
\textbf{MCC}                                                   & $\frac{1}{\abs{L}}  \cdot  \sum_{l \in L} \frac{ TP_l \cdot TN_l - FP_l\cdot FN_l}{\sqrt{(TP_l + FP_l)\cdot(TP_l+FN_l)\cdot(TN_l+FP_l)\cdot(TN+FN)}}$                & {[}-1,1{]}                                                           & -1             & 1               \\ \hline
\textbf{GMEAN}                                                 & $\frac{1}{\abs{L}} \cdot \sum_{l \in L} \sqrt{\frac{TN_l \cdot TP_l}{(TN_l+FP_l)\cdot(TP_l+FP_l)}}$                                                                                                       & {[}0,1{]}                                                            & 0              & 1               \\ \hline
\textbf{CErr}                                                  & $\frac{1}{\abs{n}}  \cdot  \sum_{i = 1}^n \mathcal{I}(\hat{y}_i \neq y_i)$                                                  & {[}0,1{]}                                                            & 1              & 0               \\ \hline
\end{tabular}%
}
\caption{Classifier evaluation metrics information. Macro-averaging is performed in each metric to generalize it to a multi-class classification problem. This approach is valid even for the case of binary classification. }
\label{tb:metrics}
\end{table}

On the other hand, for more than two classes the problem is named multi-class classification. A common way to address the classifiers' assessment in this setting is to consider a One-vs-All configuration. It consists of taking one class as the positive class, considering the remaining classes as negative ones, and repeating this process for every class. In this case, it is necessary to make metric generalizations to a multi-class environment. A popular method is averaging the metrics through micro-averaging or macro-averaging. Micro-averaging considers all elements $TP_l, TN_l, FP_l$ and $FN_l$ to extend the two-class metric equations. In contrast, macro-averaging considers the per-class metrics' performance. In this paper, a macro-averaging is performed for each metric. Table~\ref{tb:metrics} summarizes this computation. \par

Those metrics are computed for every iteration of the repeated R-Fold strategy (\ref{app:rcv}).

\section{Confusion Matrices}\label{app:matrices}

A confusion matrix is a specific table layout that allows the visualization of the performance of a supervised learning algorithm. Each row of the matrix represents the instances in a real class, while each column represents the instances in a predicted class (or vice versa). The vectors $ Y, \hat{Y} $ are used to construct each confusion matrix. \par
For each dataset, 5 confusion matrices were created, one matrix for each classifier. The three corresponding matrices of TDABC-R, TDABC-M, and TDABC-A proposed methods will be placed on the first row. The remaining two matrices will correspond to the baseline algorithms k-NN and wk-NN, and they will be arranged as a second row.

\begin{figure}[H]
\begin{center}
\includegraphics[width=0.80\textwidth]{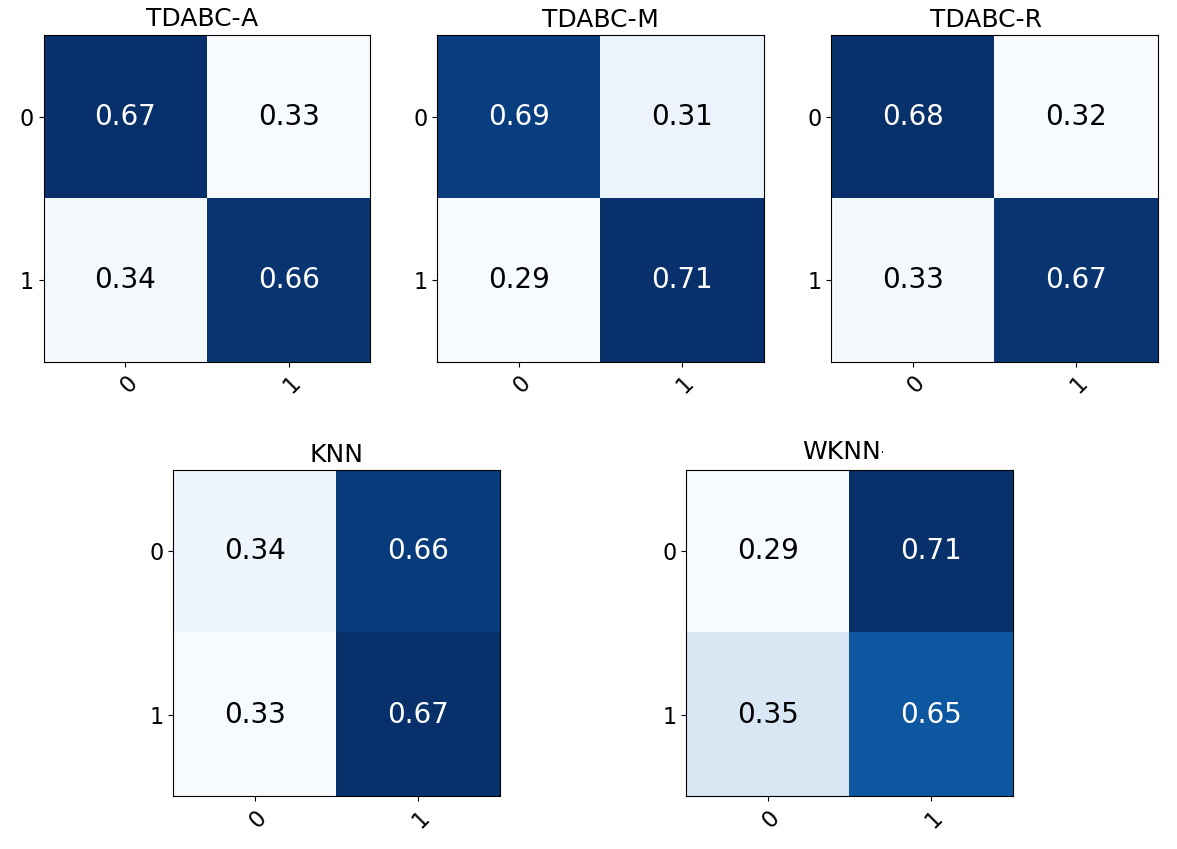}
\caption{ Confusion matrices from classifiers' results on the Circles dataset. The filtered simplicial complexes used by TDABC classifiers were built up to a maximal dimension $q=8$.}\label{fig:ALL_circle_CM0}    
\end{center}
\end{figure}

\begin{figure}[H]
\begin{center}
\includegraphics[width=0.80\textwidth]{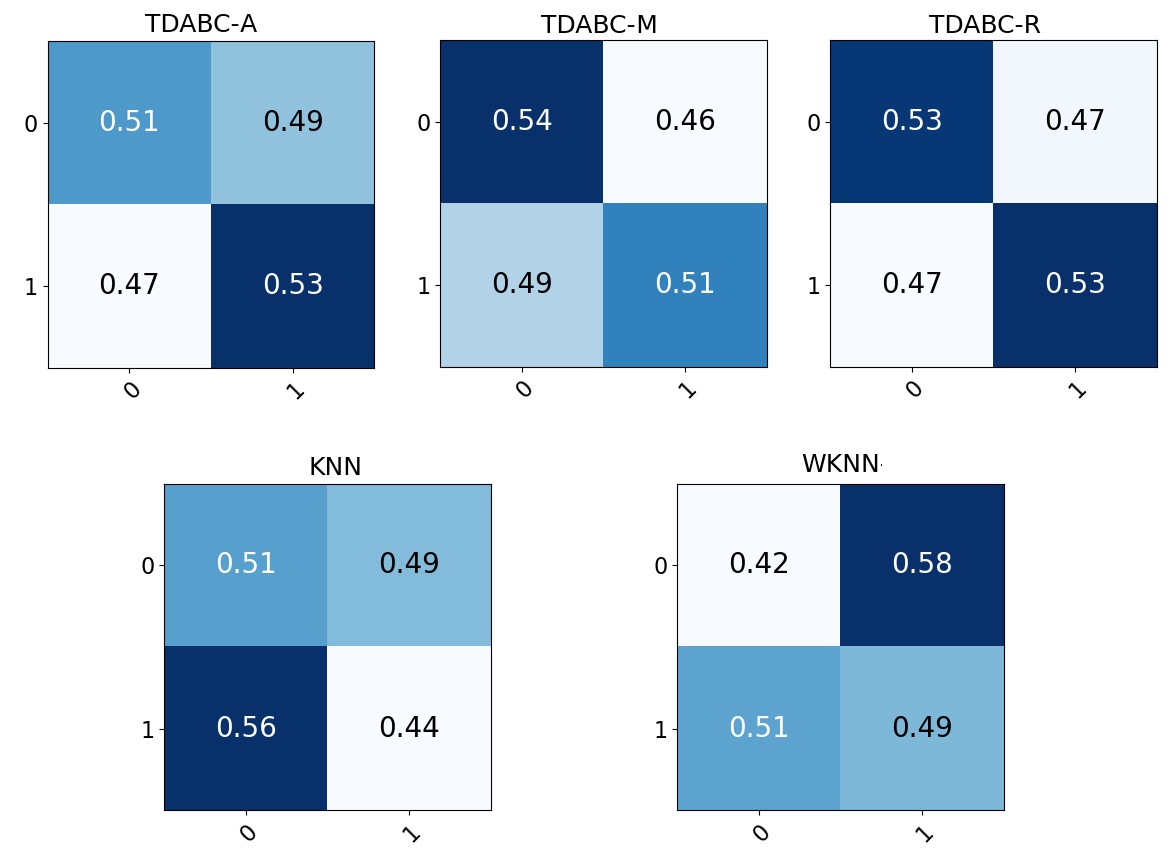}
\caption{Confusion matrices from classifiers' results on the  Moon dataset. The filtered simplicial complexes used by TDABC classifiers were built up to a maximal dimension $q=3$.}\label{fig:ALL_moon_CM0}    
\end{center}
\end{figure}

\begin{figure}[H]
\begin{center}
\includegraphics[width=0.80\textwidth]{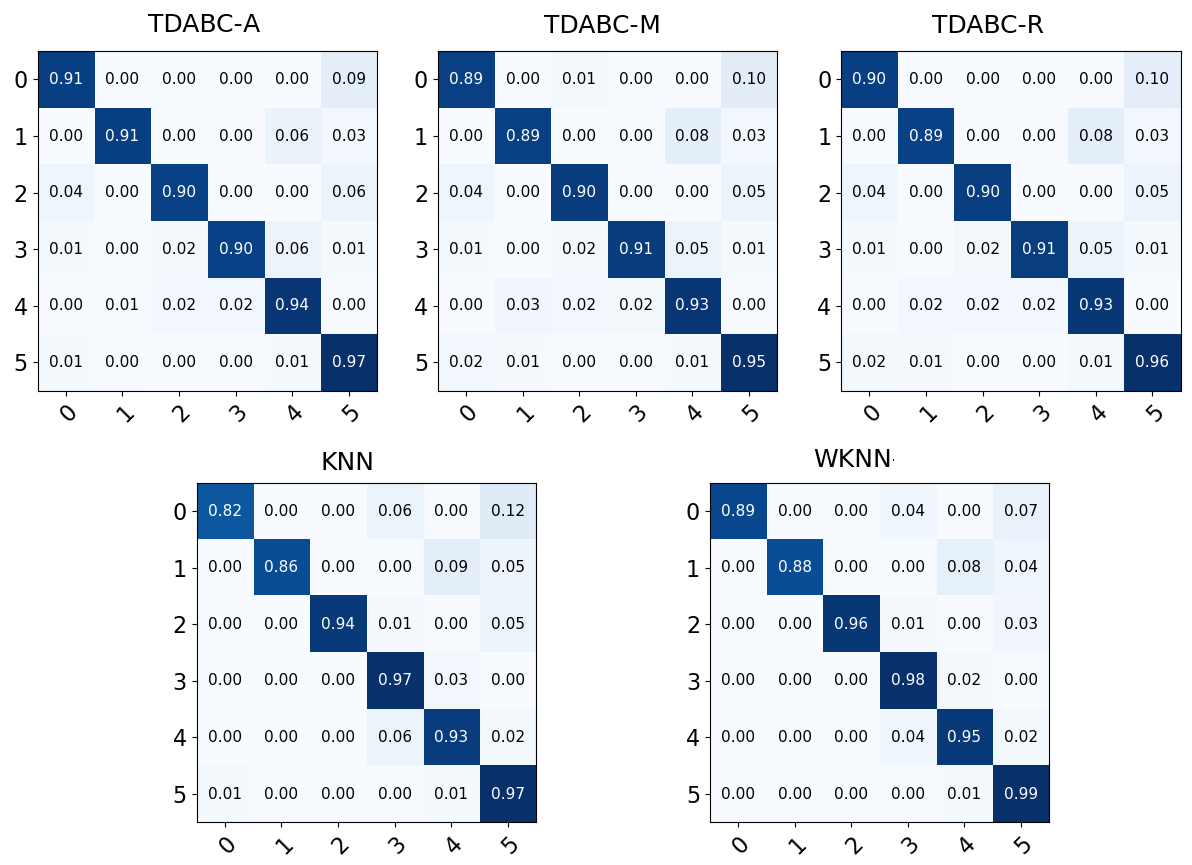}
\caption{Confusion matrices from classifiers' results on the Swissroll dataset. The filtered simplicial complexes used by TDABC classifiers were built up to a maximal dimension $q=6$.}\label{fig:ALL_swissroll_CM0}    
\end{center}
\end{figure}

\begin{figure}[H]
\begin{center}
\includegraphics[width=0.80\textwidth]{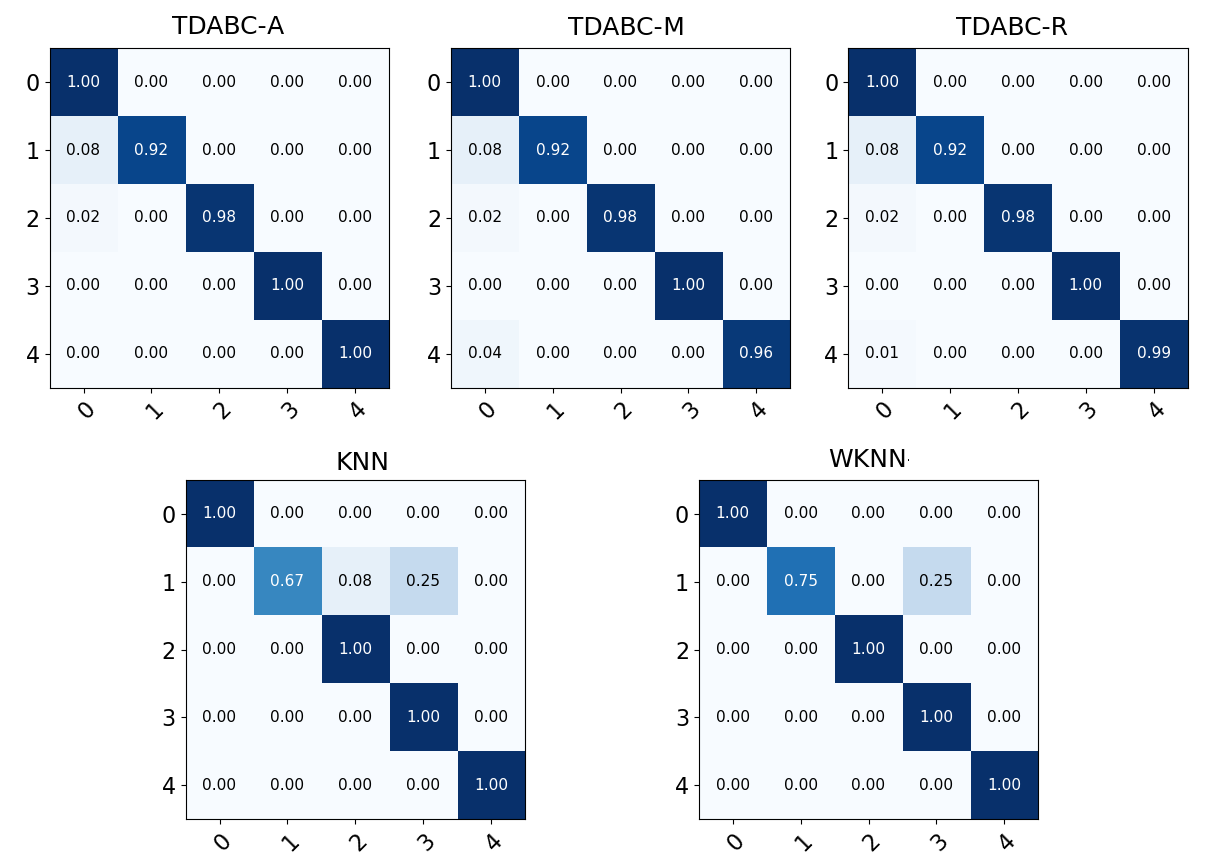}
\caption{Confusion matrices from classifiers' results on the Normdist dataset. The filtered simplicial complexes used by TDABC classifiers were built up to a maximal dimension $q=7$.}\label{fig:NORMAL_ALL_CM0}    
\end{center}
\end{figure}
\begin{figure}[H]
\begin{center}
\includegraphics[width=0.80\textwidth]{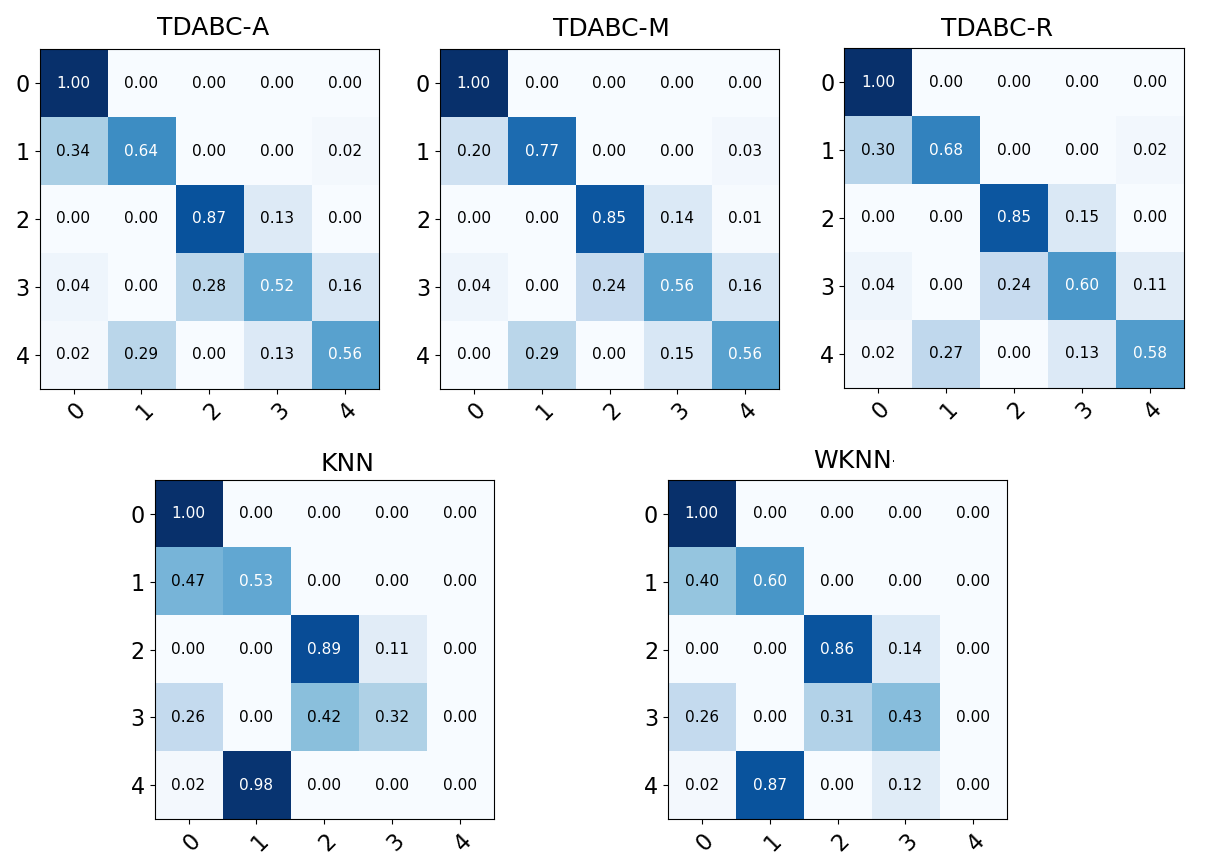}
\caption{Confusion matrices from classifiers' results on the Sphere dataset. The filtered simplicial complexes used by TDABC classifiers were built up to a maximal dimension $q=8$.}\label{fig:ALL_CM0}    
\end{center}
\end{figure}
\begin{figure}[H]
\begin{center}
\includegraphics[width=0.80\textwidth]{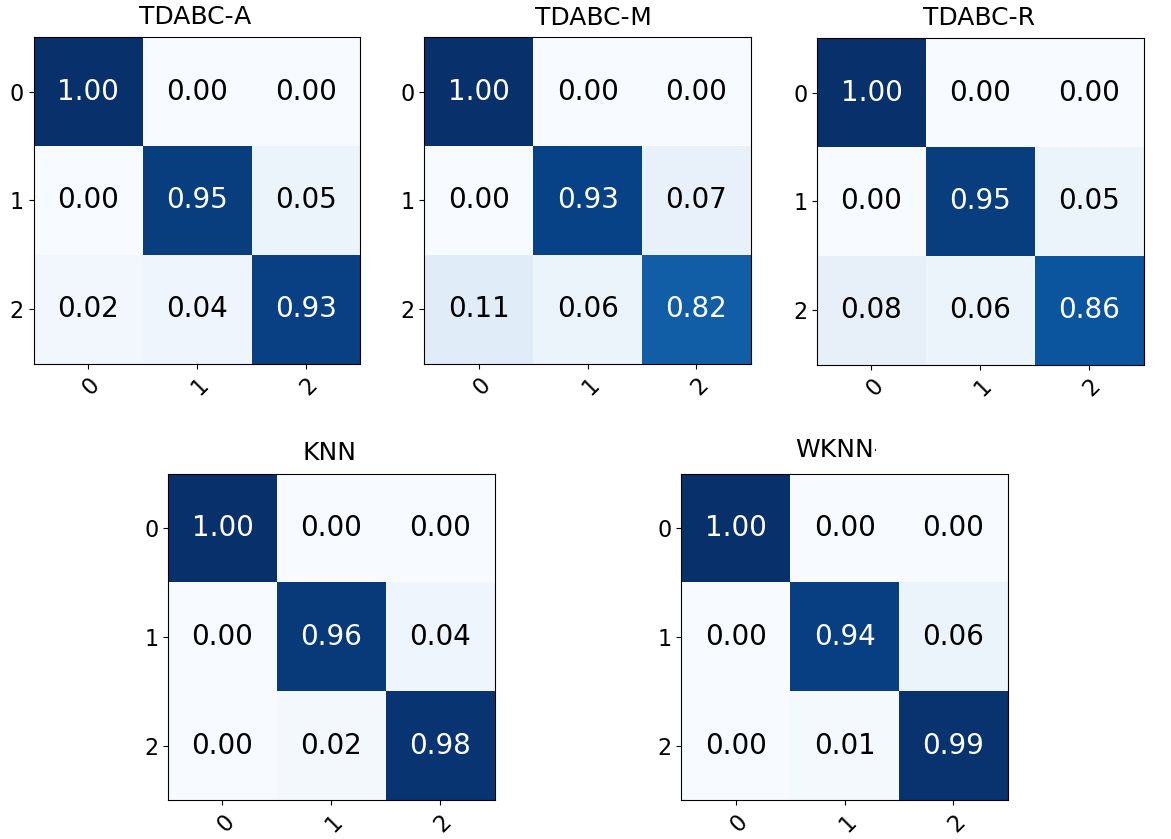}

\caption{Confusion matrices from classifiers' results on the Iris dataset. The filtered simplicial complexes used by TDABC classifiers were built up to a maximal dimension $q=8$.}\label{fig:iris_CM0}    
\end{center}
\end{figure}

\begin{figure}[H]
\begin{center}
\includegraphics[width=0.80\textwidth]{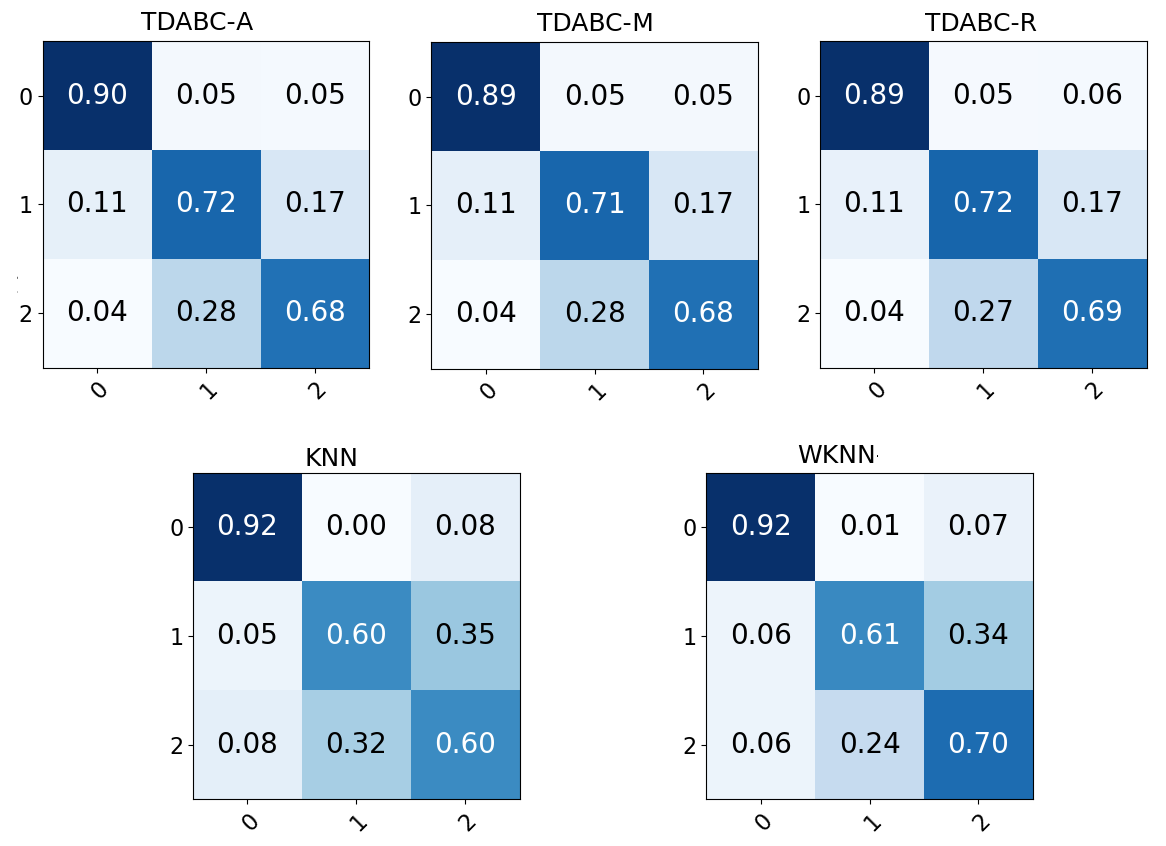}
\caption{ Confusion matrices from classifiers' results on the Wine dataset. The filtered simplicial complexes used by TDABC classifiers were built up to a maximal dimension $q=6$.}\label{fig:ALL_wine_CM0}    
\end{center}
\end{figure}
\begin{figure}[H]
\begin{center}
\includegraphics[width=0.80\textwidth]{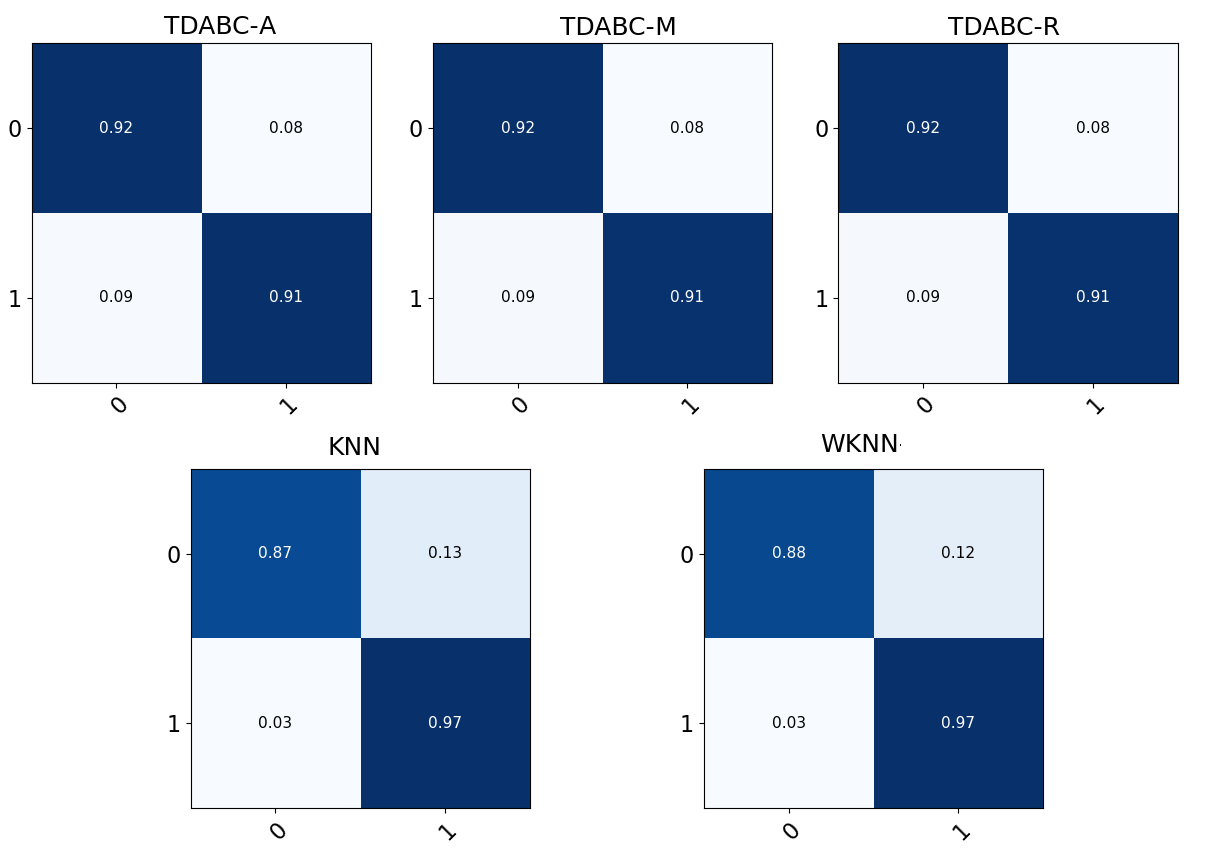}
\caption{Confusion matrices from classifiers' results on the Breast Cancer dataset. The filtered simplicial complexes used by TDABC classifiers were built up to a maximal dimension $q=4$.}\label{fig:ALL_breast_cancer_CM0}    
\end{center}
\end{figure}






\bibliography{references.bib}
\end{document}